\documentclass[journal]{IEEEtran}

\usepackage{amssymb}
\usepackage{amsmath}
\usepackage{graphicx}
\usepackage{color}
\usepackage{arydshln}
\usepackage{pifont}
\usepackage{enumerate}
\usepackage[ruled]{algorithm}
\usepackage{hyperref} 
\usepackage{algorithmic}
\usepackage{threeparttable} 
\usepackage{subcaption}
\usepackage{booktabs} 
\usepackage{multirow}
\usepackage{stfloats}

\newtheorem{theorem}{\bf \emph{Theorem}}[section]

\newtheorem{remark}{\bf \emph{Remark}}
\usepackage[numbers,sort&compress]{natbib}  %

\begin{document}

\title{{Co-Diffusion: An Affinity-Aware Two-Stage Latent Diffusion Framework for Generalizable Drug-Target Affinity Prediction}}

\author{\parbox{\linewidth}{\centering
Yining Qian, Pengjie Wang, Yixiao Li, An-Yang Lu, Cheng Tan, Shuang Li, and Lijun Liu$^{*}$
}%
\thanks{This work has been submitted to the IEEE for possible publication. Copyright may be transferred without notice, after which this version may no longer be accessible.}%
\thanks{Yining Qian is with the School of Computer Science and Engineering, Northeastern University, Shenyang 110819, China (e-mail: qianyiningning@126.com).}%
\thanks{Pengjie Wang, Yixiao Li, and An-Yang Lu are with the College of Information Science and Engineering, Northeastern University, Shenyang 110819, China (e-mail: 3508288313@qq.com, 2470845@stu.neu.edu.cn).}%
\thanks{Cheng Tan is with Westlake University, Hangzhou 310000, China (e-mail: chengtan9907@gmail.com).}%
\thanks{Shuang Li is with the School of Artificial Intelligence, Beihang University, Beijing 100000, China (e-mail: shuangliai@buaa.edu.cn).}%
\thanks{Lijun Liu is with the Key Laboratory of Bioresource Research and Development of Liaoning Province, College of Life and Health Sciences, Northeastern University, Shenyang 110169, China ( Corresponding author e-mail: liulijun@mail.neu.edu.cn).}%
}



\maketitle

\begin{abstract}

Predicting drug–target affinity is fundamental to virtual screening and lead optimization. However, existing deep models often suffer from representation collapse in stringent cold-start regimes, where the scarcity of labels and domain shifts prevent the learning of transferable pharmacophores and binding motifs. In this paper, we propose Co-Diffusion, a novel affinity-aware framework that redefines DTA prediction as a constrained latent denoising process to enhance generalization. Co-Diffusion employs a two-stage paradigm: Stage I establishes an affinity-steered latent manifold by aligning drug and target embeddings under an explicit supervised objective, ensuring that the latent space reflects the intrinsic binding landscape. Stage II introduces modality-specific latent diffusion as a stochastic perturb-and-denoise regularizer, forcing the model to recover consistent affinity semantics from noisy structural representations. This approach effectively mitigates the reconstruction-regression conflict common in generative DTA models. Theoretically, we show that Co-Diffusion maximizes a variational lower bound on the joint likelihood of drug structures, protein sequences, and binding strength. Extensive experiments across multiple benchmarks demonstrate that Co-Diffusion significantly outperforms state-of-the-art baselines, particularly yielding superior zero-shot generalization on unseen molecular scaffolds and novel protein families—paving a robust path for in silico drug prioritization in unexplored chemical spaces.

\end{abstract}

\begin{IEEEkeywords}
Latent diffusion model, drug-target affinity prediction, protein-ligand binding affinity, cold-start
\end{IEEEkeywords}

\IEEEpeerreviewmaketitle

\section{Introduction}\label{introduction}

Predicting drug--target binding affinity (DTA) is a foundational task in computer-aided drug discovery. It quantifies the strength of the interaction between a small molecule and a protein target, providing essential data for virtual screening, target validation, and lead optimization \cite{TNNLS1,TNNLS2}. Unlike binary drug--target interaction (DTI) classification, DTA provides continuous affinity estimates that enable fine-grained ranking and prioritization of therapeutic candidates \cite{nature1,DTI2025,hisif}. While experimental assays such as ELISA or surface plasmon resonance offer high-fidelity measurements, they remain labor-intensive and low-throughput. Given that the journey from initial discovery to clinical approval often exceeds a decade with multi-billion-dollar investments \cite{global2,CF-DTI}, computational DTA serves as a critical high-throughput triage layer to prioritize compounds for expensive in vitro validation \cite{PTAMY,multimodal}, as illustrated in Fig.~\ref{pipeline}.

\begin{figure}[t]
	\centering
	\includegraphics[width=7cm,height=4cm]{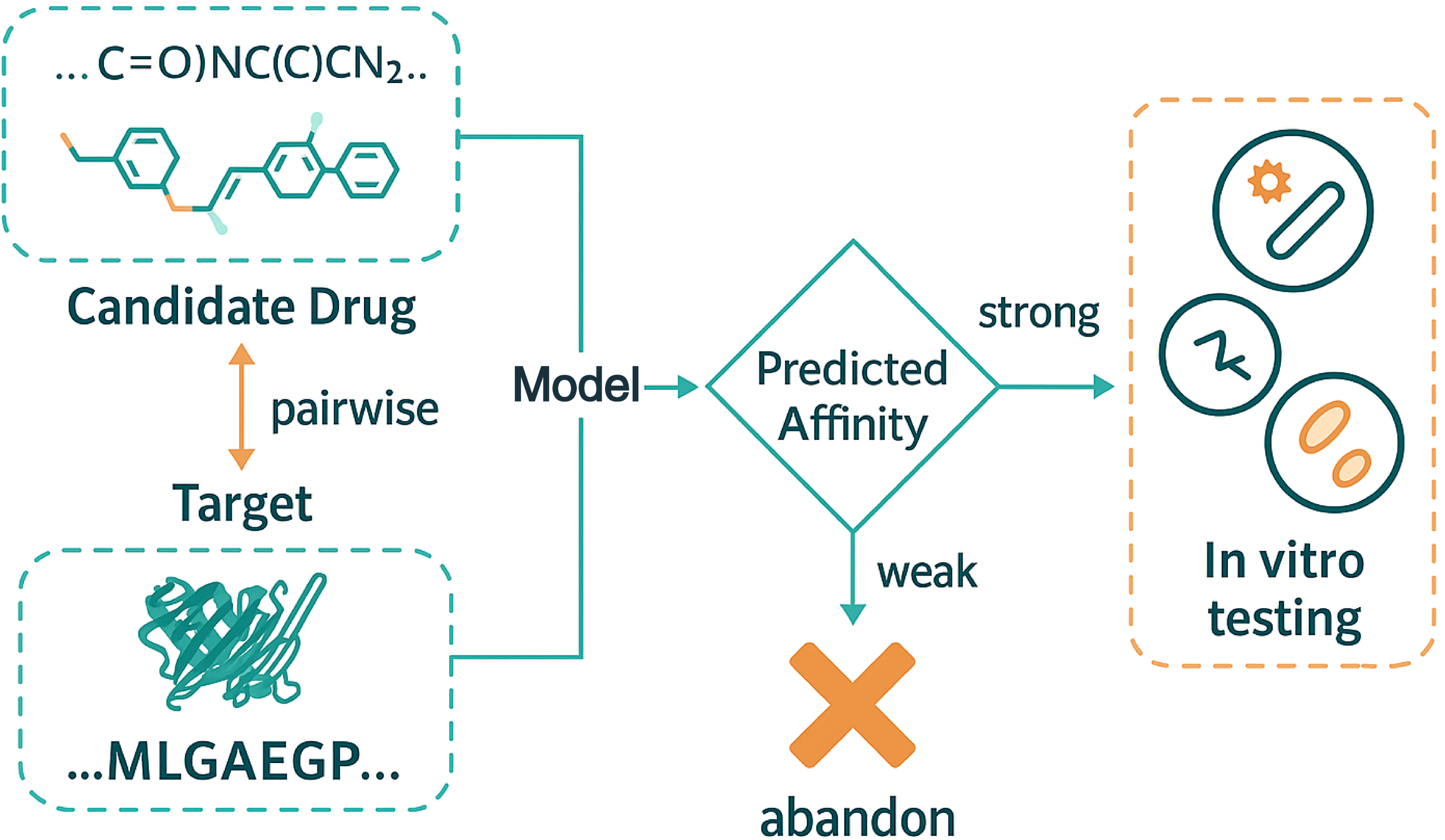}
	\caption{\small Overview of the drug discovery process. The proposed Co-Diffusion framework operates at the predictive modeling stage, providing robust affinity estimates to guide the selection of candidates before in vitro testing.}
	\label{pipeline}
\end{figure}

In practice, affinity is reported via standardized biochemical readouts such as the dissociation constant $K_d$, inhibition constant $K_i$, or $\mathrm{IC}{50}$. Because these values typically span multiple orders of magnitude, log-transformed targets (e.g., $y=-\log_{10}K_d$) are commonly employed to stabilize the learning process \cite{globalandmeasures}. DTA is thus naturally formulated as a regression problem, requiring representations that can resolve subtle chemical and structural variations to yield stable, well-calibrated predictions.

Recent deep learning models have substantially advanced supervised DTA prediction by learning end-to-end representations and capturing complex nonlinear interaction patterns \cite{intromethods,DTI2024,ColdDTA,review2023}. However, the empirical success of these models often evaporates in stringent cold-start scenarios, where test-time drugs or targets occupy unseen regions of the chemical and biological manifolds \cite{nature2025}. This generalization gap stems from a fundamental limitation: purely discriminative architectures often learn brittle, pair-specific correlations rather than intrinsic, transferable binding determinants (e.g., conserved pharmacophores or binding site motifs). Consequently, cold-start DTA calls for a stronger inductive bias that can reconcile multimodal structural diversity with affinity-relevant consistency.


These challenges have motivated the exploration of generative and probabilistic formulations for DTA. Variational approaches, such as VAE-based designs \cite{co-vae,pair-vae,TransVAE}, introduce latent variables to regularize representation learning. However, these methods frequently suffer from semantic dilution, where the structural reconstruction objective overwhelms the subtle signals required for affinity prediction, leading to a "reconstruction--regression conflict." While Latent Diffusion Models (LDMs) offer superior expressive power in capturing complex, multi-modal distributions, standard diffusion training is driven by denoising fidelity rather than binding supervision. Directly integrating diffusion into a DTA pipeline may thus yield diverse latents that remain weakly coupled to the actual binding strength.

To bridge these gaps, we propose \textbf{Co-Diffusion}, a novel affinity-aware latent diffusion framework with a two-stage training paradigm tailored for robust cold-start generalization. In Stage~I, Co-Diffusion establishes an affinity-steered latent manifold by aligning drug and target embeddings under an explicit supervised objective, ensuring the latent space serves as a semantic anchor for binding landscapes. In Stage~II, we introduce modality-specific latent diffusion as a stochastic perturb-and-denoise regularizer. By constraining the denoising dynamics to preserve affinity predictiveness, Co-Diffusion forces the model to recover consistent binding semantics from noisy structural perturbations. This approach leverages the expressive prior-learning capabilities of diffusion while ensuring that learned features remain invariant to the distribution shifts inherent in prospective discovery.

Our main contributions are summarized as follows.

\begin{itemize}
    \item We propose Co-Diffusion, an affinity-steered latent diffusion framework that harmonizes structural representation learning with binding-strength supervision, specifically optimized for stringent cold-start DTA prediction.

    \item We introduce a unique two-stage training paradigm that first anchors an affinity-aligned latent manifold and subsequently applies latent-space diffusion as a noise-robust regularizer, effectively bypassing the traditional reconstruction--regression conflict.

    \item We provide a principled probabilistic derivation showing that Co-Diffusion optimizes a variational lower bound on the joint likelihood of drug structures, protein sequences, and affinity. Extensive experiments demonstrate its state-of-the-art generalization on unseen drugs and novel targets.
\end{itemize}

\section{Related Work}

The evolution of DTA prediction has transitioned from classical statistical methods to deep learning paradigms. Existing literature can be broadly categorized into discriminative architectures and generative latent-variable frameworks.

\subsection{Discriminative Paradigms for DTA}

Discriminative models aim to learn a direct mapping from drug-target pairs to affinity values, focusing on the architectural fusion of chemical and biological modalities.

\paragraph{Structural and Sequential Encoders.}
Early milestones such as DeepDTA~\cite{DeepDTA} and WideDTA~\cite{WideDTA} utilized 1D convolutional neural networks (CNNs) to extract local motifs from SMILES strings and amino-acid sequences. To better preserve the spatial topology of small molecules, graph-based methods like GraphDTA~\cite{GraphDTA}, FAPE-DTI~\cite{FAPE-DTI} and GSAML-DTA~\cite{GSAML-DTA} employed Graph Neural Networks (GNNs) to capture geometric pharmacophores. More recently, Transformer-based models, including MolTrans~\cite{MolTrans}, AttentionDTA~\cite{AttentionDTA} and MocFormer~\cite{MocFormer}, have leveraged self-attention mechanisms to model long-range dependencies and sub-structural cross-interactions.

\paragraph{The Generalization Bottleneck.}
Despite achieving high accuracy on randomized splits, these discriminative models often suffer from representational brittleness in cold-start scenarios. Without explicit structural regularization, these models tend to memorize spurious correlations specific to the training set rather than distilling transferable binding determinants. This lack of robustness necessitates an inductive bias that can decouple essential binding semantics from distributional noise.

\subsection{Generative and Probabilistic Regularization}

To mitigate label scarcity and domain shift, researchers have introduced generative components to regularize the latent space.

\paragraph{Variational Latent-Variable Models.}
Generative DTA models primarily utilize Variational Autoencoders (VAEs) to learn structured representations. Co-VAE~\cite{co-vae} uses dual-branch VAEs with co-regularization to align drug and target latent spaces. PAIR-VAE~\cite{pair-vae} synthesizes auxiliary interaction instances by rebuilding drug--target pairs and optimizes them jointly with original data under a label-sharing protocol. TransVAE-DTA~\cite{TransVAE} combines Transformer architectures with variational layers to steer representations toward affinity-relevant clusters. However, these VAE-style designs often encounter a reconstruction--regression conflict: the heavy objective of reconstructing raw molecular structures or protein sequences can dilute the subtle affinity signals, leading to a semantic misalignment where the latent space prioritizes structural fidelity over predictive utility.

\paragraph{Emergence of Diffusion Models in Bio-AI.}
Diffusion models have recently revolutionized molecular docking (e.g., Re-Dock~\cite{Re-Dock}) and protein folding by modeling complex conformational distributions. However, their application in DTA remains largely untapped. Existing diffusion models are predominantly designed for generative sampling of coordinates or graphs, rather than serving as discriminative regularizers for affinity regression. Our \textit{Co-Diffusion} framework fills this gap by repurposing latent diffusion as a stochastic perturb-and-denoise regularizer, bridging the expressive power of generative priors with the precision of affinity-aware supervision.

\section{Methodology}

This section details the Co-Diffusion framework through five integrated components. We first formalize the DTA prediction task and provide the theoretical prerequisites for Latent Diffusion Models (LDMs). We then present a rigorous probabilistic derivation of the Co-Diffusion objective, followed by an exposition of the network architecture and the two-stage optimization strategy designed to decouple latent alignment from generative refinement.

\subsection{Problem Definition}
Let \( D \) be the space of drug compounds (SMILES) and \( T \) be the space of protein targets (amino acid sequences). A DTA dataset is defined as $\mathcal{S} = \{(\mathbf{x}_d^i, \mathbf{x}_t^i, y^i)\}_{i=1}^N$, where $\mathbf{x}_d \in \mathcal{D}$, $\mathbf{x}_t \in \mathcal{T}$, and $y \in \mathbb{R}$ represents the log-transformed affinity (e.g., $pK_d$). Our goal is to learn a predictive mapping $f: \mathcal{D} \times \mathcal{T} \rightarrow \mathbb{R}$ that generalizes to unseen drug-target pairs in cold-start scenarios.


\subsection{Latent Diffusion Model}

Latent Diffusion Models \cite{LDM} operate on a compressed latent manifold rather than raw input space, significantly enhancing computational efficiency and focus on semantic features.

\begin{figure*}[htbp]
	\centering
	\includegraphics[width=16cm,height=7cm]{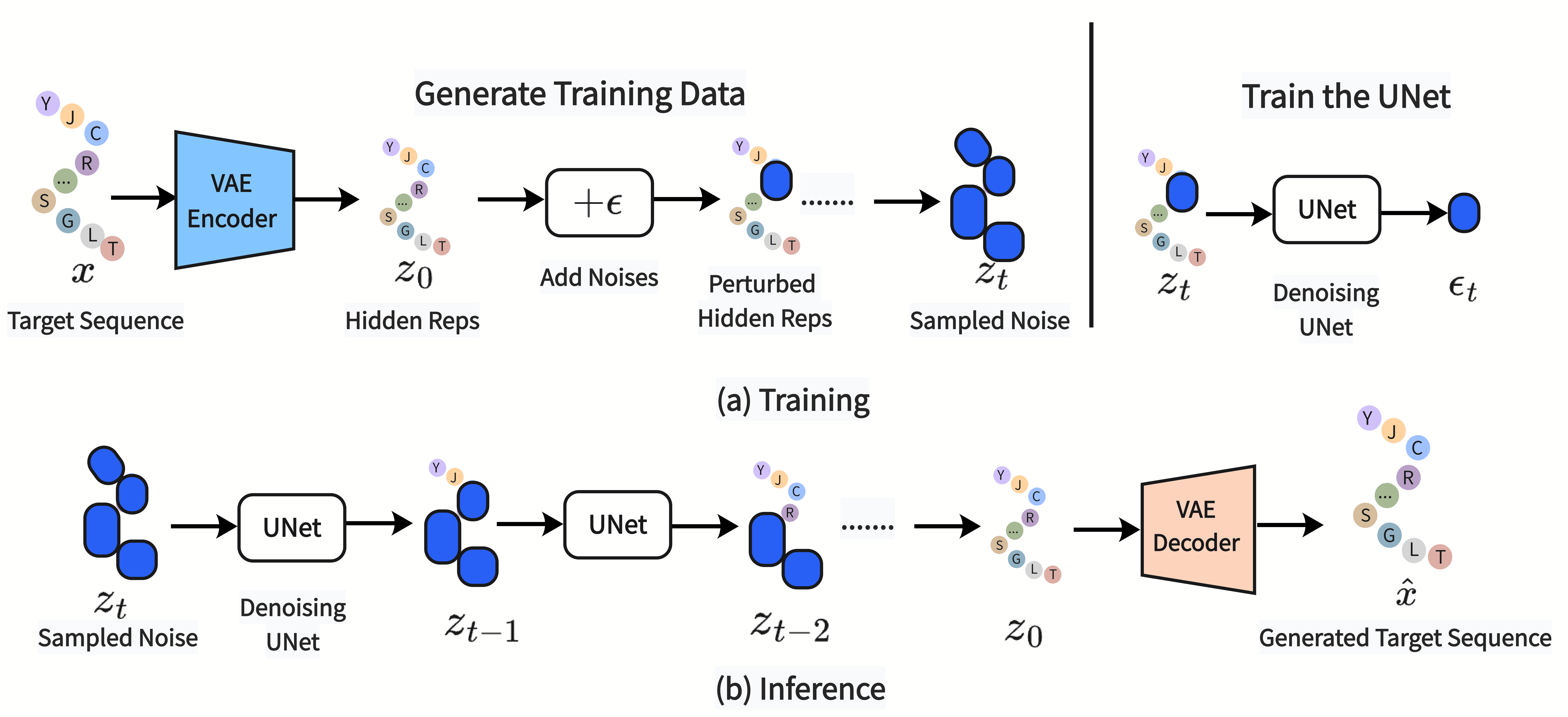}
	\caption{\textbf{Latent Diffusion Process in Training and Inference.} (a) In training, the VAE encoder produces latent $\mathbf{z}_0$, which is perturbed to $\mathbf{z}_t$ by adding noise. A UNet is trained to predict the noise, following the \textit{Noise-Prediction Training Objective}. (b) In inference, the reverse denoising process transforms sampled noise $\mathbf{z}_T$ back to $\mathbf{z}_0$, which is decoded to generate the target sequence, as described in \textit{Generation}.}
	\label{overall1}
\end{figure*}

\paragraph{Framework Overview.}
Given an observed data sample $\mathbf{x} \in \mathbb{R}^D$, a latent autoencoder is first trained to learn a compact representation $\mathbf{z}_0 \in \mathbb{R}^d$ with $d \ll D$. The autoencoder consists of an encoder $\mathcal{E}: \mathbf{x} \rightarrow \mathbf{z}_0$ and a decoder $\mathcal{D}: \mathbf{z}_0 \rightarrow \mathbf{\hat{x}}$, optimized via a reconstruction objective:
\begin{equation*}
	\mathcal{L}_{\mathrm{VAE}} = \mathbb{E}_{\mathbf{x} \sim p_{\mathrm{data}}(\mathbf{x})} \left[ | \mathbf{x} - \mathcal{D}(\mathcal{E}(\mathbf{x})) |^2 \right].
\end{equation*}

After training the autoencoder, a diffusion model (DM) is applied in the latent space $\mathbf{z}_0 = \mathcal{E}(\mathbf{x})$ to learn its generative distribution. Fig.~\ref{overall1} illustrates the overall latent diffusion framework using the target sequence as an example.

\paragraph{Latent Diffusion Process.}

Following the denoising diffusion probabilistic model (DDPM) \cite{DDPM} framework, a forward diffusion process in latent space gradually corrupts $\mathbf{z}_0$ into noise through a Markov chain:
\begin{equation*}
	q(\mathbf{z}_{1:T}|\mathbf{z}_0) = \prod_{k=1}^{T} q(\mathbf{z}_k|\mathbf{z}_{k-1})
\end{equation*}
where 
$q(\mathbf{z}_k|\mathbf{z}_{k-1}) = \mathcal{N}(\mathbf{z}_k; \sqrt{1 - \beta_k} \mathbf{z}_{k-1}, \beta_t \mathbf{I})$,
$\{\beta_k\}_{k=1}^T$ is a fixed variance schedule. As $T \rightarrow \infty$, $q(\mathbf{z}_T) \rightarrow \mathcal{N}(\mathbf{0}, \mathbf{I})$.

To learn the generative distribution, a reverse denoising process is modeled:
\begin{equation*}
	p_\theta(\mathbf{z}_{0:T}) = p(\mathbf{z}_T) \prod_{k=1}^{T} p_\theta(\mathbf{z}_{k-1}|\mathbf{z}_k)
\end{equation*}
where $p(\mathbf{z}_T) = \mathcal{N}(\mathbf{0}, \mathbf{I})$, each $p_\theta(\mathbf{z}_{k-1}|\mathbf{z}_k)$ is parameterized as a Gaussian distribution with mean predicted by a neural network $\boldsymbol{\mu}_\theta(\mathbf{z}_k, k)$ and variance $\sigma_k^2$.

\paragraph{Variational Lower Bound.}

To optimize the model, diffusion models aim to maximize the log-likelihood of data $\log p_\theta(\mathbf{z}_0)$. However, direct optimization is intractable, so an evidence lower bound (ELBO) is instead maximized by introducing the forward process $q(\mathbf{z}_{1:T}|\mathbf{z}_0)$ as an approximate posterior. This leads to the following ELBO:
\begin{equation}\label{pDM}
	\log p_\theta(\mathbf{z}_0) \geq \mathbb{E}_{q(\mathbf{z}_{1:T}|\mathbf{z}_0)} \left[ \log \frac{p_\theta(\mathbf{z}_{0:T})}{q(\mathbf{z}_{1:T}|\mathbf{z}_0)} \right] = -\mathcal{L}_{\mathrm{DM}}
\end{equation}
where $\mathcal{L}_{\mathrm{DM}}$ can be rewritten as:
\begin{equation}\label{LDM}
	\begin{aligned}
		\mathcal{L}_{\mathrm{DM}} = & D_{\mathrm{KL}}\left(q(\mathbf{z}_T|\mathbf{z}_0) \| p(\mathbf{z}_T)\right)\\
		&+ \sum_{k=2}^{T} D_{\mathrm{KL}}\left(q(\mathbf{z}_{k-1}|\mathbf{z}_k, \mathbf{z}_0) \| p_\theta(\mathbf{z}_{k-1}|\mathbf{z}_k) \right)\\
		&- \mathbb{E}_{q} [\log p_\theta(\mathbf{z}_0 | \mathbf{z}_1)].
	\end{aligned}
\end{equation}
Each term in this decomposition has a distinct meaning:
\begin{enumerate}
	\item The first term penalizes the mismatch between the terminal forward distribution $q(\mathbf{z}_T|\mathbf{z}_0)$ and the prior $p(\mathbf{z}_T)$.
	\item The summation accounts for mismatches between the reverse generative distribution $p_\theta(\mathbf{z}_{k-1}|\mathbf{z}_k)$ and the exact reverse posterior $q(\mathbf{z}_{k-1}|\mathbf{z}_k, \mathbf{z}_0)$.
	\item The final term corresponds to the reconstruction log-likelihood at step $k=1$.
\end{enumerate}

\paragraph{Noise-Prediction Training Objective.}
To simplify optimization, LDMs adopt the reparameterization proposed in \cite{DDPM}, where the forward process is expressed as:
\begin{equation*}
	\mathbf{z}_k = \sqrt{\bar{\alpha}_k} \mathbf{z}_0 + \sqrt{1 - \bar{\alpha}_k} \boldsymbol{\epsilon}, \quad \boldsymbol{\epsilon} \sim \mathcal{N}(\mathbf{0}, \mathbf{I}),
\end{equation*}
with $\bar{\alpha}_k = \prod_{s=1}^{k} (1 - \beta_s)$. The model learns to predict the noise $\boldsymbol{\epsilon}$ from the noisy latent $\mathbf{z}_k$:
\begin{equation}\label{LSIMeq1}
	\mathcal{L}_{\mathrm{simple}} = \mathbb{E}_{\mathbf{z}_0, \boldsymbol{\epsilon}, k} \left[ \left| \boldsymbol{\epsilon} - \boldsymbol{\epsilon}_\theta(\mathbf{z}_k, k) \right|^2 \right].
\end{equation}

\paragraph{Generation.}
After training, new data samples are generated by sampling $\mathbf{z}_T \sim \mathcal{N}(\mathbf{0}, \mathbf{I})$ and iteratively applying the learned reverse process to obtain $\mathbf{z}_0$, followed by $\mathcal{D}(\mathbf{z}_0)$.

LDMs effectively combine the variational principles of VAEs with the high-quality generation of DM. By operating in a learned latent space, LDMs provide a scalable and modular framework for efficient high-dimensional generation with a well-defined ELBO on the data likelihood.

\subsection{The Co-Diffusion Model}

The LDM is a state-of-the-art for generative models to find out the marginal distribution $p(\mathbf{x})$ for a high-dimensional random vector $\mathbf{x}$. However, in our scenario, it could not be used directly to predict drug-target binding affinity. In this section, referring to the structure of LDM, we propose a Co-Diffusion model to predict affinities through removing the decoder layer and adding a prediction layer.

We first formulate our drug-target affinity prediction problem as follows. Let $\mathbf{x}_d$ and $\mathbf{x}_t$ be the feature vectors for a drug and a target, and $y$ be the affinity values. We aim to learn a predictive model $p(y | \mathbf{x}_d, \mathbf{x}_t)$ based on the training triplets $\{ \mathbf{x}_d^i, \mathbf{x}_t^i, y^i \}_{i=1}^N$. We assume the latent variables for drug $\mathbf{x}_d$ and target $\mathbf{x}_t$ as $\mathbf{z}_{d,0}$ and $\mathbf{z}_{t,0}$, respectively. And
a forward diffusion process in latent space gradually corrupts $\mathbf{z}_{d,0}$ into noise
$\mathbf{z}_{d,k} = \sqrt{\bar{\alpha}_k} \mathbf{z}_{d,0} + \sqrt{1 - \bar{\alpha}_k} \boldsymbol{\epsilon}$ ($k\in[1,T]$), so does $\mathbf{z}_{t,k}$.

We propose a Co-Diffusion model as shown in Fig.~\ref{over}. The Co-Diffusion model contains two recognition models $q_{\phi_d}(\mathbf{z}_{d,0}|\mathbf{x}_d)$ and $q_{\phi_t}(\mathbf{z}_{t,0}|\mathbf{x}_t)$, which independently produce the latent variables $\mathbf{z}_{d,0}$ and $\mathbf{z}_{t,0}$ from the input $\mathbf{x}_d$ and $\mathbf{x}_t$, respectively, two diffusion processes $p_{\theta_d}(\mathbf{z}_{d,0}|\mathbf{z}_{d,k})$ and $p_{\theta_t}(\mathbf{z}_{t,0}|\mathbf{z}_{t,k})$, which produce $\mathbf{z}_{d,k}$, $\hat{\mathbf{z}}_{d,0}$ and $\mathbf{z}_{t,k}$, $\hat{\mathbf{z}}_{t,0}$, respectively, and a prediction model $p_{\theta_y}(y|\hat{\mathbf{z}}_{d,0},\hat{\mathbf{z}}_{t,0})$, which generates a drug-target affinity prediction from $\hat{\mathbf{z}}_{d,0}$ and $\hat{\mathbf{z}}_{t,0}$. 
Besides, we make the following assumptions to obtain a tractable ELBO while keeping drug–target interaction modeled in the likelihood.

\noindent\textbf{(A1) Factorized prior over latent diffusion trajectories.}
\begin{equation*}
p_{\theta}\!\left(z_{d,0:T},\,z_{t,0:T}\right)
= p_{\theta_d}\!\left(z_{d,0:T}\right)\;
  p_{\theta_t}\!\left(z_{t,0:T}\right).
\end{equation*}

\noindent\textbf{(A2) Mean–field variational family.}
\begin{equation*}
q_{\phi}\!\left(z_{d,0:T},\,z_{t,0:T}\mid x_d,x_t\right)
= q_{\phi_d}\!\left(z_{d,0:T}\mid x_d\right)\;
  q_{\phi_t}\!\left(z_{t,0:T}\mid x_t\right).
\end{equation*}

\noindent\textbf{(A3) Conditionally independent diffusion trajectories given their initial latents.}
The latent path $z_{\star,1:T}$ evolves as a Markov chain driven by its own independent noise process,
\begin{equation*}
q\!\left(z_{d,1:T},z_{t,1:T}\mid z_{d,0},z_{t,0}\right)
= q\!\left(z_{d,1:T}\mid z_{d,0}\right)\;
  q\!\left(z_{t,1:T}\mid z_{t,0}\right).
\end{equation*}

\begin{remark}
Assumptions (A1)–(A3) are modeling/variational choices made purely for tractable inference rather than biological claims. We factorize the prior and the variational family (A1–A2) and let the two diffusion trajectories evolve with independent noise given their own initial latents (A3), which mirrors the graphical model and enables stable, scalable optimization of two DDPM-style branches. The drug–target interaction is retained explicitly in the likelihood $p_{\theta_y}\!\left(y \mid z_{d,0},\,z_{t,0}\right)$, which couples the latents during training and prediction. This separation (simple priors/posteriors for efficiency, interaction in the likelihood for fidelity) follows standard probabilistic practice and reduces gradient variance and computational cost.
\end{remark}

Based on the above assumptions, we have the following theorem for the lower bound of log likelihood of the joint distribution of $(y,z_{d,0},z_{t,0})$, which is taken as the objective function of the Co-Diffusion model.

\begin{figure}[htbp]
	\centering
	\includegraphics[width=9cm,height=4.5cm]{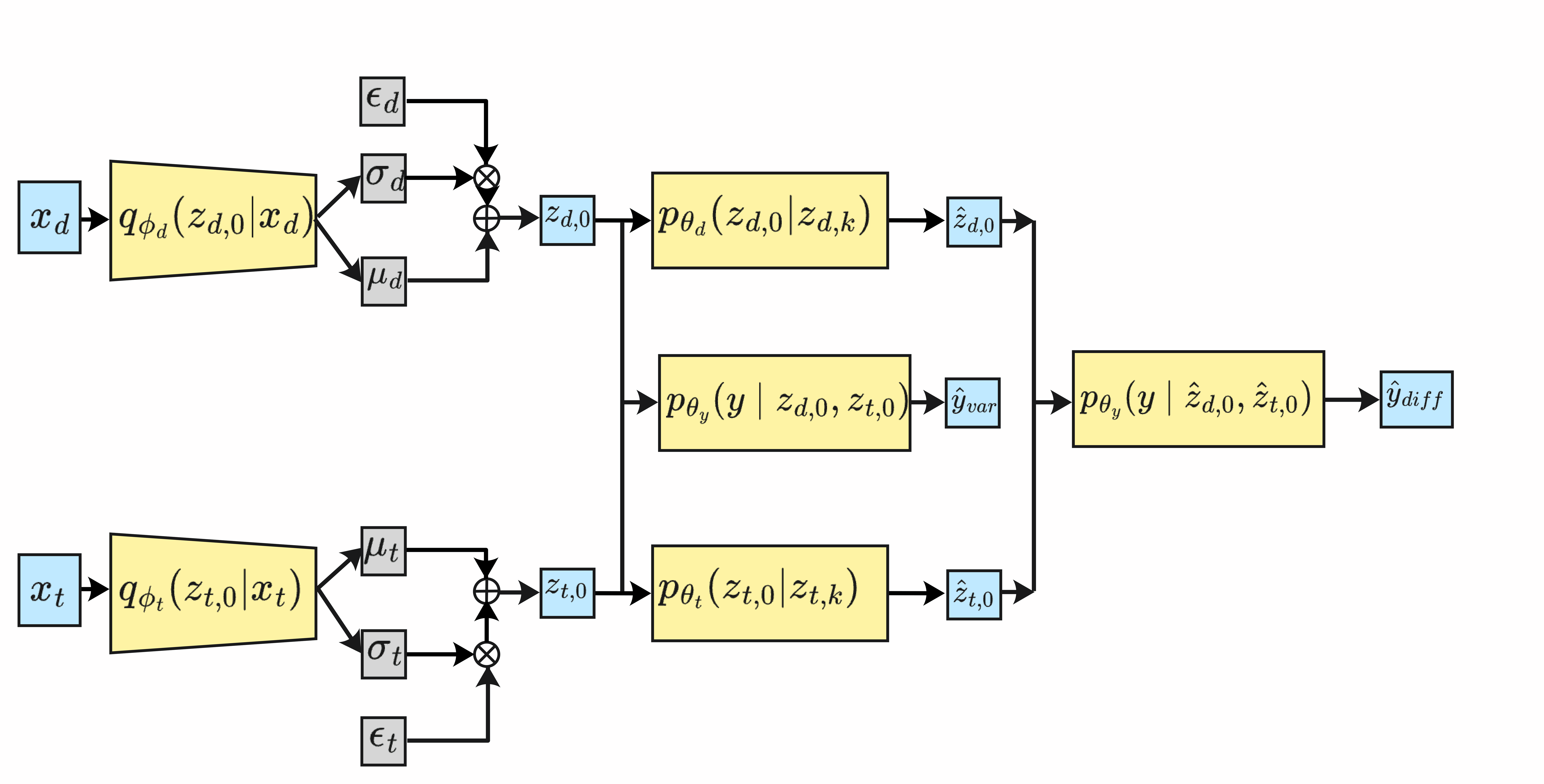}
	\caption{The graphical structure for Co-Diffusion model.}
	\label{over}
\end{figure}

\begin{theorem}\label{thm1}
	With the assumptions (A1)-(A3), the graphical model shown in Fig. \ref{over} generates an ELBO of log likelihood of the joint distribution of $(y,z_{d,0},z_{t,0})$ as
	\begin{equation}\label{thm1e1}
		\begin{aligned}
			&\log p_\theta(y,z_{d,0},z_{t,0})
			\geq \mathcal{L}(\theta; y,z_{d,0},z_{t,0}) =\mathcal{L}_\text{CoDiff}\\
			&= \mathcal{L}_\text{CoREG}
			-\mathcal{L}_\text{DrugDiff}-\mathcal{L}_\text{TargetDiff}
		\end{aligned}
	\end{equation}
	where $\theta = \{\theta_d, \theta_t, \theta_y\}$,
	\begin{align*}
		&\mathcal{L}_\text{CoREG}=\mathbb{E}_{q}[ \log p_{\theta_y}(y | \mathbf{z}_{d,0}, \mathbf{z}_{t,0})]\\
		&\mathcal{L}_\text{DrugDiff}=	D_{\mathrm{KL}}\left(q(\mathbf{z}_{d,T}|\mathbf{z}_{d,0}) \| p_{\theta_d}(\mathbf{z}_{d,T})\right)\\
		&\quad + \sum_{k=2}^{T} D_{\mathrm{KL}}\left(q(\mathbf{z}_{d,k-1}|\mathbf{z}_{d,k}, \mathbf{z}_{d,0}) \| p_{\theta_d}(\mathbf{z}_{d,k-1}|\mathbf{z}_{d,k}) \right)\\
		&\quad - \mathbb{E}_{q} \left[\log {p_{\theta_d}(\mathbf{z}_{d,0} | \mathbf{z}_{d,1})}\right],\\
		&\mathcal{L}_\text{TargetDiff}=	D_{\mathrm{KL}}\left(q(\mathbf{z}_{t,T}|\mathbf{z}_{t,0}) \| p_{\theta_t}(\mathbf{z}_{t,T})\right)\\
		&\quad + \sum_{k=2}^{T} D_{\mathrm{KL}}\left(q(\mathbf{z}_{t,k-1}|\mathbf{z}_{t,k}, \mathbf{z}_{t,0}) \| p_{\theta_t}(\mathbf{z}_{t,k-1}|\mathbf{z}_{t,k}) \right)\\
		&\quad - \mathbb{E}_{q} \left[\log {p_{\theta_t}(\mathbf{z}_{t,0} | \mathbf{z}_{t,1})}\right].
	\end{align*}
\end{theorem}
\begin{IEEEproof}
	Similar to (\ref{pDM}), the log likelihood of the joint distribution of $(y,z_{d,0},z_{t,0})$ has the following lower bound:
	\begin{equation}\label{logpxy}
		\begin{aligned}
			\log& p_\theta(y,z_{d,0},z_{t,0}) \\
			= &\log \int p_\theta(y, \mathbf{z}_{d,0:T}, \mathbf{z}_{t,0:T}) \, d\mathbf{z}_{d,1:T} \, d\mathbf{z}_{t,1:T}\\
			\geq& \mathbb{E}_{q}\left[ \log \frac{p_\theta(y,\mathbf{z}_{d,0:T},\mathbf{z}_{t,0:T})}
			{q(\mathbf{z}_{d,1:T},\mathbf{z}_{t,1:T}|\mathbf{z}_{d,0},\mathbf{z}_{t,0})} \right].
		\end{aligned}
	\end{equation}

	Under Assumptions (A1)-(A3), and the graphical model as shown in Fig. \ref{over}, we have:
	\begin{equation*}
		\begin{aligned}
			&q(\mathbf{z}_{d,1:T},\mathbf{z}_{t,1:T}|\mathbf{z}_{d,0},\mathbf{z}_{t,0})=
			q(\mathbf{z}_{d,1:T}|\mathbf{z}_{d,0})
			q(\mathbf{z}_{t,1:T}|\mathbf{z}_{t,0}),\\
			&p_\theta(y, \mathbf{z}_{d,0:T}, \mathbf{z}_{t,0:T})= 
			p_{\theta}(\mathbf{z}_{d,0:T})
			p_{\theta}(\mathbf{z}_{t,0:T})p_{\theta}(y|\mathbf{z}_{d,0}, \mathbf{z}_{t,0}).
		\end{aligned}
	\end{equation*}
	Substituting above equations into (\ref{logpxy}) provides
	\begin{equation}\label{LcoLDM}
		\begin{aligned}
			&\log p_\theta(y,z_{d,0},z_{t,0})\ge \mathbb{E}_q \left[\log p_{\theta_y}(y | \mathbf{z}_{d,0}, \mathbf{z}_{t,0})\right]\\
			& + \mathbb{E}_{q} \left[\log \frac{p_{\theta_d}(\mathbf{z}_{d,0:T})}
			{q(\mathbf{z}_{d,1:T}|\mathbf{z}_{d,0})}\right]   + 
			\mathbb{E}_{q} \left[\log \frac{p_{\theta_t}(\mathbf{z}_{t,0:T})}
			{q(\mathbf{z}_{t,1:T}|\mathbf{z}_{t,0})}\right]. 
		\end{aligned}
	\end{equation}
	Referring to (\ref{pDM}), the lower bound of likelihood of the joint distribution in (\ref{logpxy}) can be obtained from (\ref{LcoLDM}).
\end{IEEEproof}

In (\ref{thm1e1}), $\mathcal{L}_{\text{DrugDiff}}$ and $\mathcal{L}_{\text{TargetDiff}}$ represent the DM bound for drugs and targets, respectively, and $\mathcal{L}_{\text{CoREG}}$ is a co-regularized term which represents the regression bound responsible for the affinity reconstruction penalty for a pair of drug and target.

For $\mathcal{L}_{\text{CoREG}}$, we further assume $p_{\theta_y}(y|{\mathbf{z}}_{d,0}, {\mathbf{z}}_{t,0})$ to be a normal distribution, whose probability are computed from ${\mathbf{z}}_{d,0}$ and ${\mathbf{z}}_{t,0}$ through the network, then we have:
\begin{equation}
	\log p_{\theta_y}(y|\mathbf{z}_{d,0}, \mathbf{z}_{t,0}) = -\lambda(y - \hat{y})^2 + C.
\end{equation}
where $\hat{y} = f_{\theta_y}(\hat{\mathbf{z}}_{d,0}, \hat{\mathbf{z}}_{t,0})$, $C$ is a constant, and $\lambda$ is the co-regularization parameter.

Using the diffusion process
$\mathbf{z}_{d,k} = \sqrt{\bar{\alpha}_k} \mathbf{z}_{d,0} + \sqrt{1 - \bar{\alpha}_k} \boldsymbol{\epsilon}_d$ ($\boldsymbol{\epsilon}_d \sim \mathcal{N}(\mathbf{0}, \mathbf{I})$), $\mathbf{z}_{t,k} = \sqrt{\bar{\alpha}_k} \mathbf{z}_{t,0} + \sqrt{1 - \bar{\alpha}_k} \boldsymbol{\epsilon}_t$ ($\boldsymbol{\epsilon}_t \sim \mathcal{N}(\mathbf{0}, \mathbf{I})$), 
similar to (\ref{LSIMeq1}),
the ELBO $\mathcal{L}_\text{CoDiff}$ in the Co-Diffusion model at point $(y,z_{d,0},z_{t,0})$ can be simplified as follows:
\begin{equation*}
	\begin{aligned}
		&\mathcal{L}_\text{CoREG}^{simple} \approx  - \lambda (y - f_{\theta_y}
		(\hat{\mathbf{z}}_{d,0}, \hat{\mathbf{z}}_{t,0}))^2\\
		&\mathcal{L}_\text{DrugDiff}^{simple} \approx - \mathbb{E}_{\mathbf{z}_{d,0}, \boldsymbol{\epsilon}_d, k} \left[ \left| \boldsymbol{\epsilon}_d - \boldsymbol{\epsilon}_{\theta_d}(\sqrt{\bar{\alpha}_k} \mathbf{z}_{d,0} + \sqrt{1 - \bar{\alpha}_k} \boldsymbol{\epsilon}_d, k) \right|^2 \right]\\
		&\mathcal{L}_\text{TargetDiff}^{simple} \approx - \mathbb{E}_{\mathbf{z}_{t,0}, \boldsymbol{\epsilon}_t, k} \left[ \left| \boldsymbol{\epsilon}_t - \boldsymbol{\epsilon}_{\theta_t}(\sqrt{\bar{\alpha}_k} \mathbf{z}_{t,0} + \sqrt{1 - \bar{\alpha}_k} \boldsymbol{\epsilon}_t, k) \right|^2 \right]
	\end{aligned} 
\end{equation*}
where $k\in[1,T]$ can be selected randomly, and for not too small $\sqrt{\bar{\alpha}_k}$, $\hat{\mathbf{z}}_{d,0}$ and $\hat{\mathbf{z}}_{t,0}$ are reconstructed by
\begin{align}\label{hz0}
	\hat{\mathbf{z}}_{*,0}
	=\frac{z_{*,k}-\sqrt{1 - \bar{\alpha}_k}\boldsymbol{\epsilon}_{\theta_*}}{\sqrt{\bar{\alpha}_k}}.
\end{align}

\subsection{Network Architecture}

As shown in Fig.~4, Co-Diffusion integrates five sequential components:
\ding{172}tokenization \& feature extraction,
\ding{173}variational encoding,
\ding{174}latent diffusion module,
\ding{175}regression on variational latents, and
\ding{176}regression on reconstructed latents.


\begin{figure}[htbp]
	\centering
	\includegraphics[width=9cm,height=13cm]{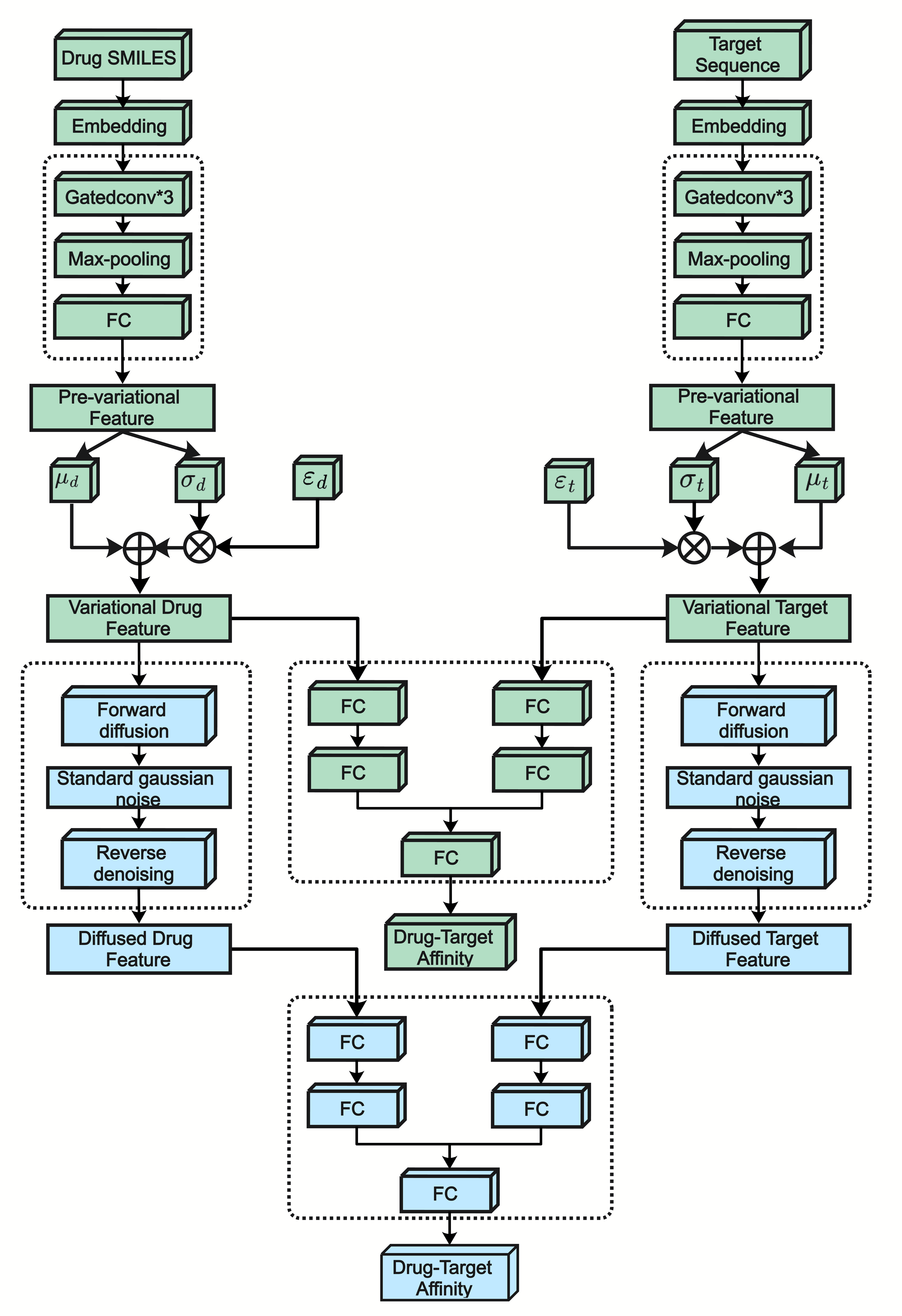}
	\caption{Overview of Co-Diffusion.} 
	\label{overall}
\end{figure}

\paragraph{Tokenization \& GatedConv Extractors.}
SMILES strings and protein sequences are first tokenized and embedded into high-dimensional vectors. These embeddings are processed through stacked Gated Convolution (GatedConv) blocks. Each block couples a 1\mbox{-}D convolution with a Gated Linear Unit (GLU), followed by layer normalization, dropout, and a residual connection. The gating mechanism adaptively suppresses uninformative noise while the residual stacking ensures stable gradient flow across long biological sequences. Finally, global max-pooling and a linear projection yield compact pre\mbox{-}variational feature vectors \(h_d\) and \(h_t\).

\paragraph{Variational Encoders}

Two lightweight neural heads transform \((h_d, h_t)\) into Gaussian parameters \((\mu, \log\sigma)\); The initial latents $\mathbf{z}_{d,0}$ and $\mathbf{z}_{t,0}$ are sampled via the reparameterization trick:
\[
z_{d,0}=\mu_d+\sigma_d\odot\varepsilon_d,\qquad
z_{t,0}=\mu_t+\sigma_t\odot\varepsilon_t,\quad
\varepsilon_{\ast}\sim\mathcal{N}(0,I).
\]
where \(\odot\) denotes element-wise multiplication and \(\varepsilon_{\ast}\) follows a standard normal distribution. These latents serve a dual purpose: they support supervised affinity prediction in Stage \ding{175} and act as clean semantic anchors for diffusion-based reconstruction in Stage \ding{174}, effectively regularizing the latent geometry.


\paragraph{Latent Diffusion Module} 
The drug and protein latents are denoised by independent $\varepsilon$-prediction networks operating in latent space. Each network adopts a 1\mbox{-}D UNet\mbox{-}style architecture with symmetric skip connections to preserve muliti-scale structural information. Diffusion timesteps are encoded using sinusoidal positional embeddings and integrated into the UNet blocks via FiLM\mbox{-}style (Feature-wise Linear Modulation) affine transformations, enabling time\mbox{-}conditioned denoising. The forward process gradually corrupts \(z_{\ast,0}\) to \(z_{\ast,k}\); the network then predicts the added noise \(\hat{\varepsilon}\) to reconstruct \(\hat{z}_{\ast,0}\). Operating in a low-dimensional latent space significantly reduces computational overhead while concentrating the model's capacity on affinity-relevant feature interactions.

\paragraph{Regression on Variational Latents (CoREG\(_{\text{var}}\))} 
A symmetric regressor head projects and concatenates $\mathbf{z}_{d,0}$ and $\mathbf{z}_{t,0}$ to output an intermediate affinity estimate $\hat{y}_{\text{var}}$. This path provides the primary supervised signal that shapes the latent manifold before any diffusion refinement is introduced, ensuring that the latent space is inherently organized by binding strength.

\paragraph{Regression on Reconstructed Latents (CoREG\(_{\text{diff}}\))} 
AFollowing the denoising process, \((\hat{z}_{d,0}, \hat{z}_{t,0})\) are fed into a twin regressor head to produce the final prediction \(\hat{y}_{\text{diff}}\). Training this head encourages the diffusion modules to recover task-critical determinants rather than generic structural artifacts, thereby enhancing the model's robustness and zero-shot generalization to unseen drug-target pairs.

\subsection{Two\mbox{-}Stage Training Strategy and Objective Functions}
\label{sec:two-stage}

\paragraph{Two\mbox{-}Stage Optimization}
To mitigate the reconstruction--regression conflict inherent in generative DTA models, we adopt a decoupled two-stage training strategy that separates affinity alignment from generative refinement.

\emph{Stage~I (\ding{172}→\ding{173}→\ding{175}):} 
We jointly optimize the GatedConv extractors and variational encoders using the affinity-supervised CoREG$_{\text{var}}$ objective (see the green path in Fig.~\ref{overall}). During this stage, the diffusion modules remain inactive. This phase establishes a binding-steered manifold, ensuring the model first captures "what determines affinity" without being distracted by structural denoising tasks.

\emph{Stage~II (\ding{173}→\ding{174}→\ding{176}):} 
The Stage-I encoders are frozen to preserve the established affinity-aligned space. We then activate the modality-specific latent diffusion branches (see the blue path in Fig.~\ref{overall}). Each latent is perturbed and denoised via the UNet-$\boldsymbol{\epsilon}$ architecture, optimized by the CoREG$_{\text{diff}}$ head. The diffused latents must simultaneously satisfy two conditions: (i) maintaining high affinity-predictive power and (ii) remaining consistent with the Stage-I anchors. This ensures that the generative prior stabilizes and regularizes the affinity manifold rather than overriding its semantic structure.

\paragraph{Objective Functions}
Stage~I focuses exclusively on the supervised regression accuracy:
\begin{equation}
  \mathcal{L}_{\text{Stage-I}} 
  \;=\; \mathcal{L}_{\text{CoREG}_{\text{var}}}^{\text{simple}}.
\end{equation}
Stage II optimizes a composite objective that couples latent-space denoising with final prediction precision:
\begin{equation}
  \mathcal{L}_{\text{Stage-II}}
  \;=\; \mathcal{L}_{\text{CoREG}_{\text{diff}}}^{\text{simple}}
  \;-\; \mathcal{L}_{\text{DrugDiff}}^{\text{simple}}
  \;-\; \mathcal{L}_{\text{TargetDiff}}^{\text{simple}}.
\end{equation}
The components are defined as:
\begin{itemize}
  \item $\mathcal{L}_{\text{DrugDiff}}^{\text{simple}}$ and $\mathcal{L}_{\text{TargetDiff}}^{\text{simple}}$ are the standard $\varepsilon$-prediction losses (latent-space denoising) that drive the reverse processes $z_{d,T}\!\rightarrow\!z_{d,0}$ and $z_{t,T}\!\rightarrow\!z_{t,0}$ across diffusion timesteps.
  \item $\mathcal{L}_{\text{CoREG}_{\text{diff}}}^{\text{simple}}$ is the regression loss on reconstructed latents, ensuring the denoising process remains task-aware.
  \item The negative signs in $\mathcal{L}_{\text{Stage-II}}$ represent the maximization of diffusion likelihood in a minimization framework, ensuring all terms are optimized in a consistent direction.
\end{itemize}

\begin{algorithm}[t]
	\caption{Two-Stage Training for Co-Diffusion}
	\label{alg:Co-Diffusion}
	\begin{algorithmic}[1]
		\REQUIRE Training set $\mathcal{D} = \{ (\mathbf{x}_d^i, \mathbf{x}_t^i, y^i) \}_{i=1}^{N}$; diffusion schedule $\{\bar{\alpha}_k\}_{k=1}^{T}$
		\ENSURE Trained parameters $\phi_d, \phi_t, \theta_d, \theta_t, \theta_y$
		\STATE \textbf{Stage I: Encoder Pretraining}
		\FOR{each epoch}
		\FOR{each mini-batch in $\mathcal{D}$}
		\STATE Compute $\boldsymbol{\mu}_d^i, \boldsymbol{\sigma}_d^i=f_{\phi_d}(x_d)$, $\boldsymbol{\mu}_t^i,\boldsymbol{\sigma}_t^i=f_{\phi_t}(x_t)$
		\STATE Sample $\mathbf{z}_{d,0}, \mathbf{z}_{t,0}$ using reparameterization
		\STATE Compute $\hat{y}_\text{var} = f_{\theta_y}(\mathbf{z}_{d,0}, \mathbf{z}_{t,0})$
		\STATE Compute loss $\mathcal{L}_{\text{Stage-I}}$
		and update $\phi_d, \phi_t, \theta_y$
		\ENDFOR
		\ENDFOR
		\STATE \textbf{Stage II: Diffusion Model Training} (Fix $\phi_d, \phi_t$)
		\FOR{each epoch}
		\FOR{each mini-batch in $\mathcal{D}$}
		\STATE Encode $\mathbf{x}_d, \mathbf{x}_t$ into $\mathbf{z}_{d,0}, \mathbf{z}_{t,0}$
		\STATE Select $k$ from $[1,T]$ randomly.
		\STATE Sample noise $\boldsymbol{\epsilon}_d$, $\boldsymbol{\epsilon}_t$
		\STATE Generate $\mathbf{z}_{d,k}, \mathbf{z}_{t,k}$ via diffusion schedule
		\STATE Predict noises $\hat{\boldsymbol{\epsilon}}_d = \boldsymbol{\epsilon}_{\theta_d}(\mathbf{z}_{d,k}, k)$, $\hat{\boldsymbol{\epsilon}}_t = \boldsymbol{\epsilon}_{\theta_t}(\mathbf{z}_{t,k}, k)$
		\STATE Compute $\hat{\mathbf{z}}_{d,0}$, $\hat{\mathbf{z}}_{t,0}$ by (\ref{hz0})
		\STATE Compute $\hat{y}_\text{diff} = f_{\theta_y}(\hat{\mathbf{z}}_{d,0}, \hat{\mathbf{z}}_{t,0})$
		\STATE Compute loss $\mathcal{L}_{\text{Stage-II}}$
		and update $\theta_d, \theta_t$, $\theta_y$
		\ENDFOR
		\ENDFOR
	\end{algorithmic}
\end{algorithm}

\section{Results}

In this section, we first describe the benchmark datasets, evaluation metrics, and experimental settings used in this study. We then report results across multiple evaluation scenarios and visualize the correspondence between predicted and ground-truth affinities. Next, ablation studies are conducted to quantify the individual contribution of each component in Co-Diffusion. Finally, we present a series of analyses, including sensitivity to key hyperparameters, visualization of the enhanced features, and out-of-sample evaluation.

\subsection{Datasets}

The benchmark datasets selected to evaluate our model were the KIBA \cite{KIBA} and Davis \cite{davis} datasets, which are widely recognized, commonly used drug-target affinity datasets. Table~\ref{dataset} illustrates the statistical information of these two datasets. Both KIBA and Davis offer drug–target affinity labels; Davis records $K_d$, whereas KIBA employs KIBA scores \cite{hisif}. For numerical stability on Davis, we apply the logarithmic transformation to obtain p$K_d$ as in Eq.~\ref{KIBA}.

\begin{equation}
	\text{p}K_d = -\log_{10}(K_d) + 9
    \label{KIBA}
\end{equation}

\begin{table}[htbp]
\centering
\renewcommand{\arraystretch}{1.2}
\caption{Statistics of the benchmark datasets used in this study.}
\begin{tabular}{lccc}
\hline
\textbf{Dataset} & \textbf{Targets} & \textbf{Drugs} & \textbf{Binding pairs} \\
\hline
KIBA  & 229  & 2111 & 118{,}254 \\
Davis & 442  & 68    & 30{,}056 \\
\hline
\end{tabular}
\label{dataset}
\end{table}

\textbf{Data Splitting} To assess the generalization performance of Co-Diffusion in practical cold-start scenarios, we adopt a systematic dataset partitioning approach:

\begin{itemize}
	\item Training Set: Contains 80\% of drugs and 80\% of targets, along with their corresponding interaction pairs. This set is used to train the model.
	\item Validation and Test Sets: Constructed from the remaining 20\% of drugs and targets to simulate cold-start conditions. We define three disjoint evaluatoin settings:
	\begin{itemize}
		\item Unseen Drugs (UD): Comprises interactions between 20\% new drugs and 80\% training targets.
		\item Unseen Targets (UT): Comprises interactions between 20\% new targets and 80\% training drugs.
		\item Unseen Pairs (UP): Comprises interactions between 20\% new drugs and 20\% new targets, forming entirely novel drug-target combinations.
	\end{itemize}
\end{itemize}
Each setting is further split evenly into validation and test subsets (50\% each), ensuring a strict separation between training and evaluation data and enabling a rigorous assessment of the model’s performance under cold-start conditions.

\textbf{Sequence Processing} Drugs and targets, represented by SMILES strings and amino acid sequences respectively, are tokenized into fixed-length indexed arrays: 
\begin{itemize}
	\item Davis: $L_d\leq 85$, $L_p\leq 1200$
	\item KIBA: $L_d\leq 100$, $L_p\leq 1000$
\end{itemize}
where $L_d$ is the SMILES length and $L_p$ is the protein length. 

Unique characters (64 for SMILES, 25 for proteins) are mapped to integers for uniform encoding.

\subsection{Evaluation Metrics}

We employ four widely recognized and representative metrics that are standard in DTA prediction: mean squared error (MSE), mean absolute error (MAE), Concordance Index (CI), and $r_m^2$. These complementary indicators jointly evaluate the predictive accuracy, ranking consistency, and overall agreement between predictions and experimental measurements, forming a comprehensive and comparable assessment framework across benchmarks and prior studies.

\paragraph{MSE} Let $y$ and $\hat{y}$ denote the ground-truth and predicted affinity scores for the test drug–target pairs, respectively. MSE measures the predictive accuracy by the average squared deviation:
\begin{equation}
  \text{MSE} = \frac{1}{n} \sum_{i=1}^{n} (\hat{y}_i - y_i)^2.
\end{equation}

\paragraph{MAE} MAE evaluates the absolute prediction error by the mean magnitude of deviations:
\begin{equation}
  \text{MAE} = \frac{1}{n} \sum_{i=1}^{n} \lvert \hat{y}_i - y_i \rvert.
\end{equation}

\paragraph{CI} CI quantifies the probability that the predicted affinity scores preserve the correct ordering of true affinities across all comparable drug–target pairs. It is defined as:
\begin{equation}
	\text{CI} = \frac{1}{Z} \sum_{\delta_i > \delta_j} \Phi(b_i - b_j)
\end{equation}
where $b_i$ and $b_j$ are predicted affinities for samples with ground-truth affinities $\delta_i$ and $\delta_j$ respectively, with $\delta_i > \delta_j$. $Z$ is the total number of valid comparison pairs. $\Phi(\cdot)$ is a piecewise function defined as:
\begin{equation}
	\Phi(x) = 
	\begin{cases}
		1, & \text{if } x > 0 \\
		0.5, & \text{if } x = 0 \\
		0, & \text{if } x < 0
	\end{cases}
\end{equation}

\begin{figure*}[htpb]
  \centering
  \subfloat[MSE and MAE]{%
    \includegraphics[width=0.47\linewidth]{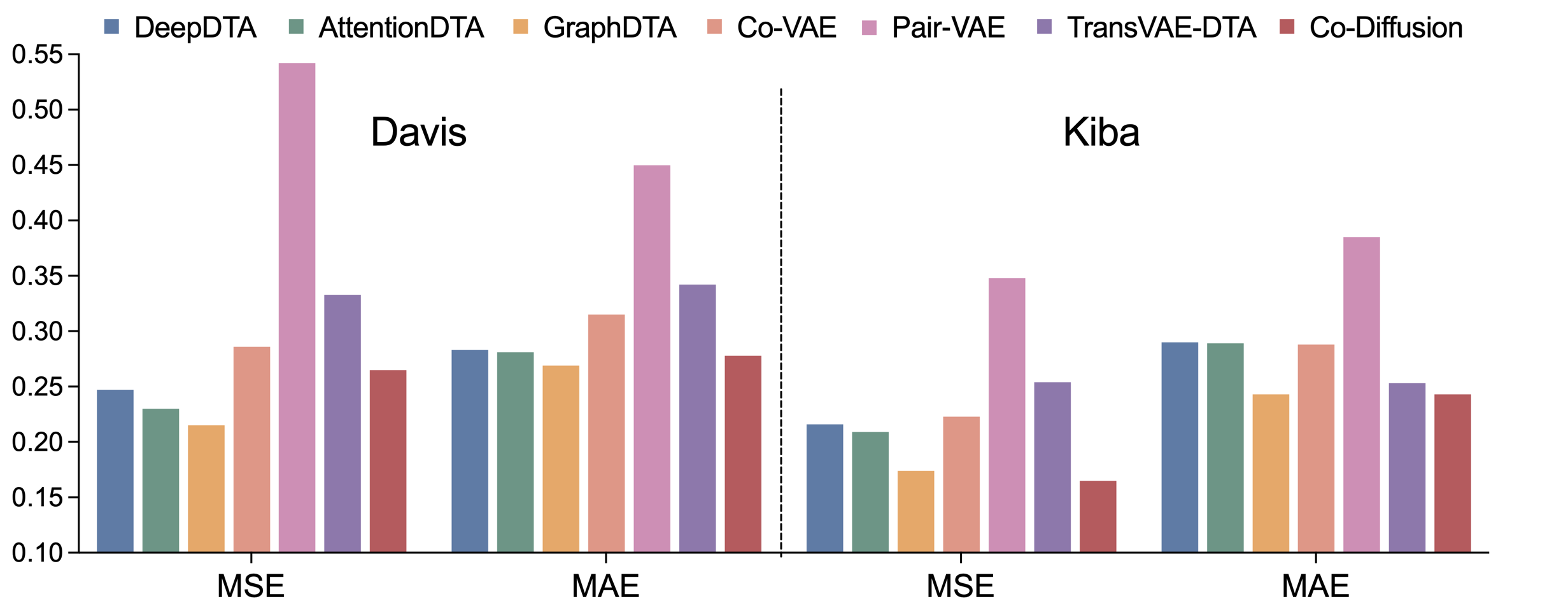}%
    \label{fig:left}
  }\hfill
  \subfloat[CI and $r_m^2$]{%
    \includegraphics[width=0.47\linewidth]{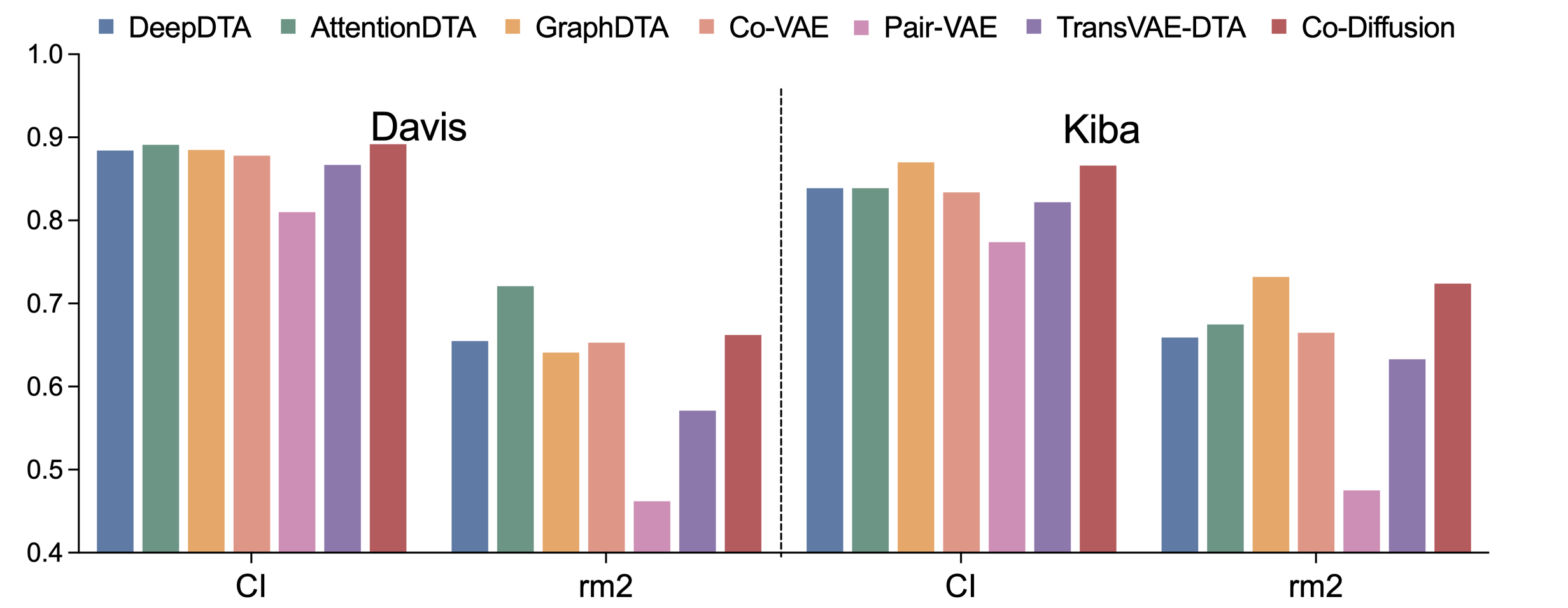}%
    \label{fig:right}
  }
  \caption{Performance comparison on Davis and KIBA under random split.}
  \label{random}
\end{figure*}

\begin{table*}[htbp]
	\centering
	\small
	\renewcommand{\arraystretch}{1.0}
	\setlength{\tabcolsep}{10pt}
	\caption{Performance Comparison on Davis and KIBA Datasets Under Cold-start Scenarios.}
	\label{results}
	\begin{tabular}{p{1.5cm} p{2cm} p{2.5cm} c c c c}
		\toprule
		\textbf{Dataset}& \textbf{Settings} & \textbf{Method} & \textbf{MSE(std)}$\downarrow$ & \textbf{MAE(std)}$\downarrow$ & \textbf{CI(std)}$\uparrow$ & $r_m^2$$\uparrow$ \\
		\midrule
		\multirow{15}{*}{\textbf{Davis}} 
		&\multirow{7}{*}{Unseen-drug} 
		&DeepDTA & 0.745 (0.008) & 0.704 (0.050) & 0.699 (0.045) & 0.078 \\
		&& AttentionDTA & 0.656 (0.004) & 0.622 (0.003) & 0.682 (0.010) & 0.104 \\
		&& GraphDTA & 0.590 (0.016) & 0.552 (0.014) & 0.695 (0.008) & 0.077\\
		&& Co-VAE & 0.577 (0.009) & 0.523 (0.014) & 0.616 (0.005) & 0.128  \\
        && TransVAE-DTA & 0.643 (0.012) & 0.542 (0.009) & 0.543 (0.010) & 0.007 \\
        && Pair-VAE & 0.570 (0.004) & 0.498 (0.007) & 0.730 (0.003) & 0.134  \\
		&& Co-Diffusion & \textbf{0.568} (0.007) & \textbf{0.471} (0.005) & \textbf{0.722} (0.007) & \textbf{0.166} \\
		\cmidrule{2-7}
		&\multirow{7}{*}{Unseen-target} 
		& DeepDTA & 0.884 (0.024) & 0.711 (0.008) & 0.751 (0.033) & 0.183 \\
		&& AttentionDTA & 0.565 (0.014) & 0.560 (0.022) & 0.814 (0.002) & 0.450  \\
		&& GraphDTA & 0.565 (0.010) & 0.527 (0.011) & 0.763 (0.003) & 0.399 \\
		&& Co-VAE & 0.570 (0.018) & 0.518 (0.016) & 0.811 (0.011) & 0.443 \\
        && TransVAE-DTA & 0.863 (0.011) & 0.543 (0.007) & 0.747 (0.007) & 0.198  \\
        && Pair-VAE & 0.637 (0.010) & 0.497 (0.009) & 0.813 (0.006) & 0.373  \\
		&& Co-Diffusion & \textbf{0.556} (0.021) & \textbf{0.429} (0.008) & \textbf{0.836} (0.008) & \textbf{0.461}  \\
		\cmidrule{2-7}
		&\multirow{7}{*}{\shortstack{Unseen-Pair}} 
		& DeepDTA & 0.886 (0.027) & 0.737 (0.012) & 0.712 (0.300) & 0.072 \\
		&& AttentionDTA & 0.837 (0.020) & 0.696 (0.002) & 0.715 (0.006) & 0.114 \\
		&& GraphDTA & 0.874 (0.036) & 0.606 (0.023) & 0.631 (0.021) & 0.032  \\
		&& Co-VAE &  0.828 (0.026) & 0.617 (0.025) & 0.597 (0.011) & 0.096   \\
        && TransVAE-DTA & 0.908 (0.005) & 0.622 (0.006) & 0.538 (0.004) & 0.008 \\
        && Pair-VAE & 0.820 (0.032) & 0.582 (0.013) &  0.741 (0.010) & 0.010  \\
		&& Co-Diffusion  & \textbf{0.806} (0.015) & \textbf{0.551} (0.013) & \textbf{0.728} (0.010) & \textbf{0.149}  \\
		\midrule
		\multirow{15}{*}{\textbf{KIBA}} 
		&\multirow{7}{*}{Unseen-drug} 
		& DeepDTA & 0.457 (0.022) & 0.461 (0.011) & 0.715 (0.007) & 0.335 \\
		&& AttentionDTA & 0.420 (0.012) & 0.444 (0.019) & 0.733 (0.008) & 0.383 \\
		&& GraphDTA & \textbf{0.400} (0.011) & 0.434 (0.016) & 0.750 (0.004) & \textbf{0.444} \\
		&& Co-VAE & 0.416 (0.004) & 0.417 (0.007) & 0.751 (0.011) &0.397 \\
        && TransVAE-DTA & 0.590 (0.009) & 0.539 (0.008) & 0.656 (0.010) & 0.189 \\
        && Pair-VAE & 0.496 (0.006) & 0.484 (0.007) & 0.719 (0.004) & 0.297  \\
		&& Co-Diffusion & 0.421 (0.003) &\textbf{0.406} (0.004) & \textbf{0.751} (0.007) & 0.402 \\
		\cmidrule{2-7}
		&\multirow{7}{*}{Unseen-target} 
		& DeepDTA & 0.377 (0.016) & 0.435 (0.013) & 0.677 (0.017) & 0.332\\
		&& AttentionDTA & 0.370 (0.045) & 0.426 (0.027) & 0.678 (0.050) & 0.337 \\
		&& GraphDTA & 0.457 (0.045) & 0.481 (0.027) & 0.627 (0.028) & 0.215\\
		&& Co-VAE & 0.353 (0.010) & 0.397 (0.005) & 0.680 (0.010) & 0.364 \\
        && TransVAE-DTA & 0.654 (0.003) & 0.600 (0.004) & 0.512 (0.007) & 0.100  \\
        && Pair-VAE & 0.411 (0.004) & 0.443 (0.009) & 0.641 (0.005) & 0.282  \\
		&& Co-Diffusion & \textbf{0.347} (0.009) &  \textbf{0.383} (0.015) & \textbf{0.685} (0.010) & \textbf{0.377}\\
		\cmidrule{2-7}
		&\multirow{7}{*}{\shortstack{Unseen-Pair}} 
		& DeepDTA & 0.493 (0.006) & 0.506 (0.013) &  0.633 (0.025) &  0.137 \\
		&& AttentionDTA & 0.480 (0.010) & 0.502 (0.022) & 0.649 (0.007) & 0.163  \\
		&& GraphDTA & 0.551 (0.054) & 0.536 (0.030) & 0.574 (0.023) & 0.045 \\
		&& Co-VAE & 0.463 (0.027) & 0.477 (0.034) & 0.649 (0.008) &  0.176 \\
        && TransVAE-DTA & 0.711 (0.008) & 0.644 (0.010) & 0.497 (0.008) & 0.025  \\
        && Pair-VAE & 0.469 (0.002) & 0.495 (0.005) & 0.635 (0.006) & 0.151  \\
		&& Co-Diffusion & \textbf{0.432} (0.005) &  \textbf{0.459} (0.008) & \textbf{0.660} (0.009) & \textbf{0.201} \\
		\bottomrule
	\end{tabular}
	\vspace{1mm}
	\begin{tablenotes}%
		\item 1.All DTA models were trained using the hyperparameter settings reported in their original papers for consistency.
        \item 2.Results reported as mean ± standard deviation over five independent runs.
	\end{tablenotes}
\end{table*}

A CI value close to 1 indicates that the predicted ranking aligns well with the true affinity order, while values around 0.5 or below suggest random or reversed rankings.

\paragraph{$r_m^2$} The $r_m^2$ metric captures the overall agreement between predicted and experimental affinity values by incorporating both the correlation strength and systematic deviation from the ideal regression line:
\begin{equation}
	r_m^2 = r^2 \left(1 - \sqrt{r^2 - r_0^2} \right)
\end{equation}
where $r^2$ is the standard coefficient of determination and $r_0^2$ is its baseline counterpart obtained by regression through the origin. Higher $r_m^2$ values (approaching 1) indicate stronger predictive reliability and better generalization performance.

\subsection{Hyperparameter Settings}

We implemented the proposed Co-Diffusion model in Python 3.9 using PyTorch 2.0.0 with CUDA 11.8. All experiments were conducted on a single NVIDIA GeForce RTX 4090 GPU (24GB memory). The model was trained for up to 100 epochs with a batch size of 256 using the Adam optimizer with a learning rate of 0.001. Dropout was applied at a rate of 0.2 in convolutional layers and 0.1 in the fully connected layers of the regression module. Drug SMILES strings and target sequences were tokenized and embedded into fixed-size matrices of shape ($L_d$, 128) and ($L_p$, 128), respectively. The encoder consists of three Gated Convolutional layers with output channels of 256, 512, and 768, each with a kernel size of 4. A global max-pooling layer is applied, followed by a fully connected layer that projects the representation to a 384-dimensional latent vector. In the diffusion module, we adopt a linear noise schedule with $\beta \in [0.0001, 0.0004]$ and set the number of diffusion steps to $T = 1000$. Both the $CoReg_{var}$ and $CoReg_{diff}$ modules comprise three fully connected layers with hidden dimensions [512, 64, 1].

\subsection{Baselines}
We compare Co-Diffusion with six state-of-the-art baselines. DeepDTA, AttentionDTA, and GraphDTA are taken as canonical representatives of sequence-based, Transformer-based, and graph-based paradigms, respectively; Co-VAE, TransVAE-DTA, and PAIR-VAE exemplify contemporary generative approaches. We briefly summarize these models below.

\begin{itemize}
	\item DeepDTA (2018)\cite{DeepDTA}: It uses two parallel 1D CNN encoders on one-hot SMILES and amino-acid sequences, then concatenates the features and regresses affinity with fully connected layers.
	\item AttentionDTA (2019)\cite{AttentionDTA}: It encodes drugs and proteins with sequence models (e.g., CNN/RNN) and uses cross-modal attention to emphasize salient substructures and residues before affinity regression.
    \item GraphDTA (2021)\cite{GraphDTA}: It represents drugs as molecular graphs and applies GNNs (GCN/GAT/GIN) while encoding proteins from sequences, then fuses both embeddings to predict affinity.
    \item Co-VAE (2021)\cite{co-vae}: It learns dual variational latents for drugs and targets, uses co-regularization to align them, and jointly optimizes reconstruction and affinity objectives to improve robustness.
    \item TransVAE-DTA (2024)\cite{TransVAE}: It couples Transformer encoders with a variational latent layer and cross-attention, leveraging reconstruction and regression losses to stabilize and guide representation learning.
    \item PAIR-VAE (2025)\cite{pair-vae}: It reconstructs drug–target pairs to generate augmented interactions under a label-sharing scheme, and jointly trains original and reconstructed pairs to enhance generalization.
    
\end{itemize}

\subsection{Performance Comparison with State-of-the-Art Approaches}

We evaluate Co-Diffusion against state-of-the-art baselines on Davis and KIBA under two distinct protocols: (i) random in distribution splits and (ii) stringent cold-start splits.

\textbf{Random splits.} Under the random split regimen, most deep models deliver uniformly high predictive fidelity (low MSE/MAE and high CI/$r_m^2$). As shown in Fig.~\ref{random}, Co-Diffusion remains highly competitive on Davis and outstrips all baselines on KIBA, achieving the superior balance between error minimization and ranking consistency. These results indicate that latent diffusion can match or exceed established architectures under standard conditions. However, prior studies \cite{co-vae,DrugBAN,TransformerCPI} have noted that random splits may introduce spurious performance inflation due to scaffold-level redundancy and sequence overlap, which allows models to exploit single-modality "shortcuts" rather than distilling genuine interaction signals. We therefore proceed to cold-start evaluations for a more rigorous assessment of model robustness.

\textbf{Cold-start splits.} Table~\ref{results} reports the performance under three cold-start scenarios: unseen drug, unseen target, and unseen pair. All methods exhibit a significant performance decay relative to the random split, with the unseen-pair scenario emerging as the most formidable challenge. On Davis, Co-Diffusion consistently surpasses competing methods across all three settings. Notably, in the unseen-drug case, it matches or improves upon the best MSE while delivering substantial gains in CI and $r_m^2$; in the unseen-target case, it attains the best MSE, CI and $r_m^2$ with a marginal MAE improvement; and  in the most challenging unseen-pair scenario, Co-Diffusion achieves a 6.4\% reduction in MAE compared to the second-best Co-VAE and a 2.6\% improvement in $r_m^2$ over AttentionDTA. On KIBA, Co-Diffusion secures the best MAE and competitive CI in the unseen-drug setting (with GraphDTA slightly leading on MSE/$r_m^2$), and surpasses all baselines across all metrics in the unseen-target and unseen-pair settings, with the margin widening in the latter. These results highlight the strong cold-start generalization of Co-Diffusion, particularly when both drugs and targets are novel. Figure \ref{truth} presents the scatter plots of predicted versus observed affinities on Davis (top) and KIBA (bottom) across the three cold-start settings.


\begin{figure*}[htpb]
	\centering
	\begin{subfigure}[b]{0.3\textwidth}
		\includegraphics[width=\textwidth]{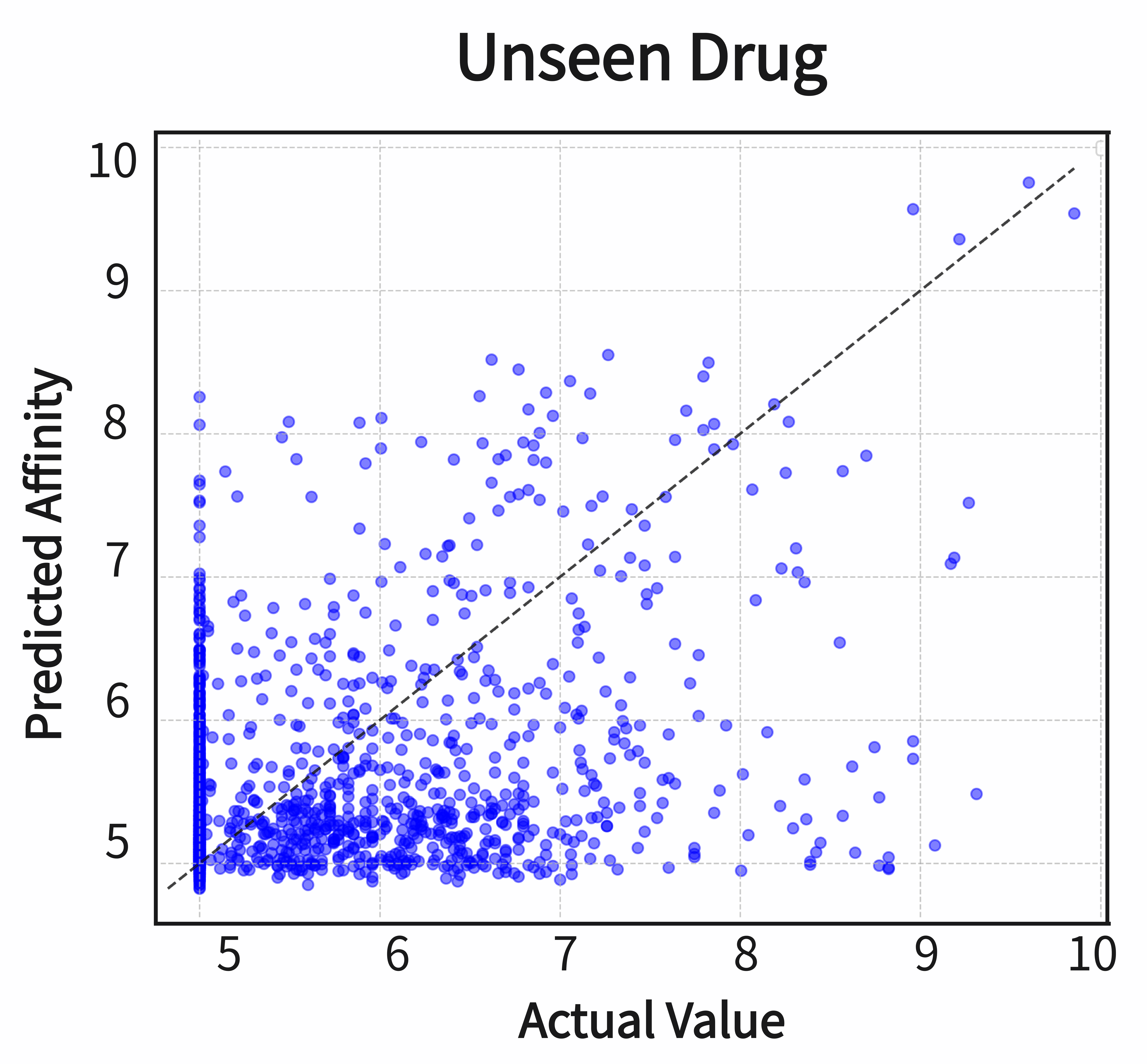}
		\label{fig:sub1}
	\end{subfigure}
	\hspace{0.5mm}
	\begin{subfigure}[b]{0.3\textwidth}
		\includegraphics[width=\textwidth]{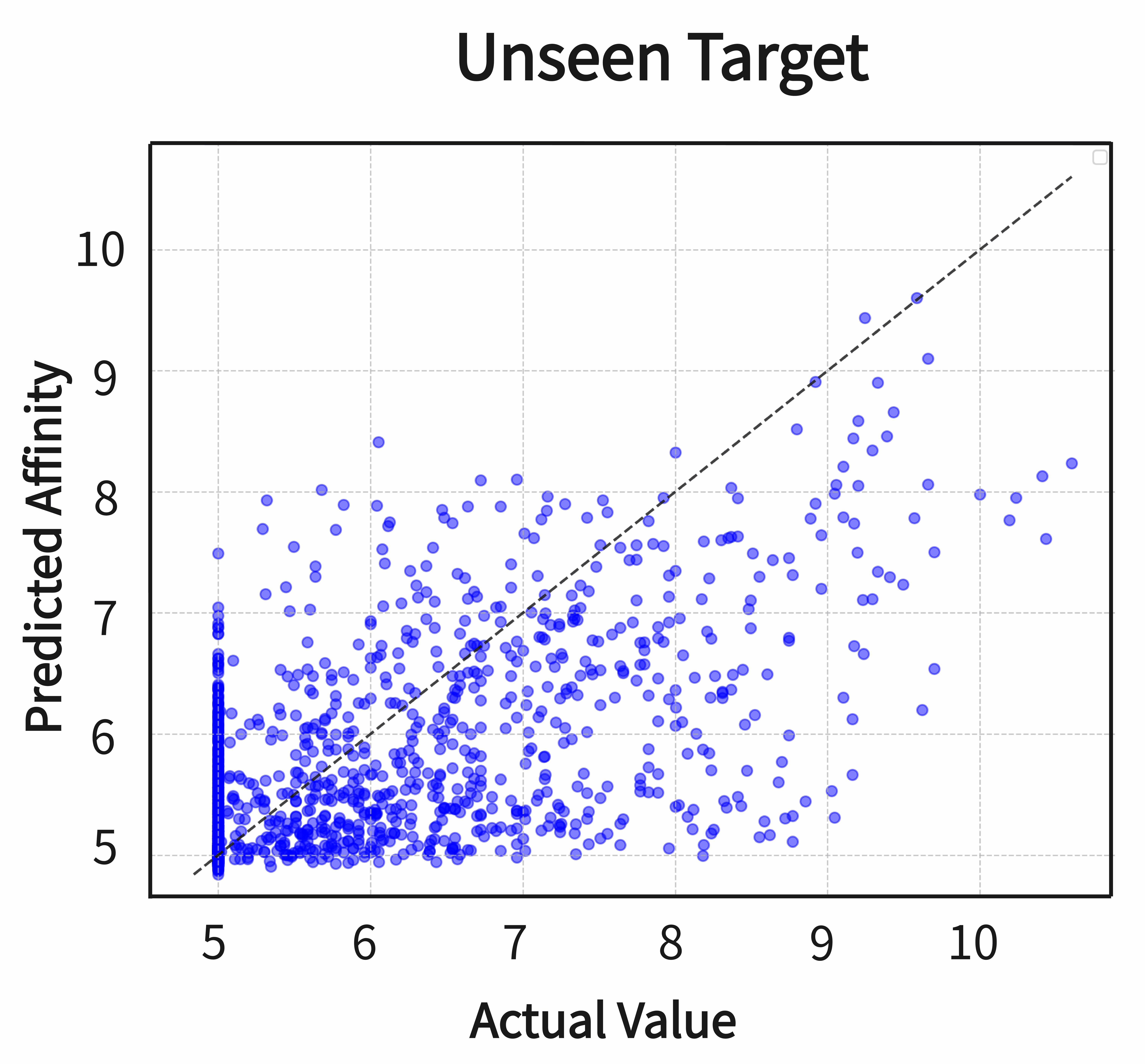}
		\label{fig:sub2}
	\end{subfigure}
	\hspace{0.5mm}
	\begin{subfigure}[b]{0.3\textwidth}
		\includegraphics[width=\textwidth]{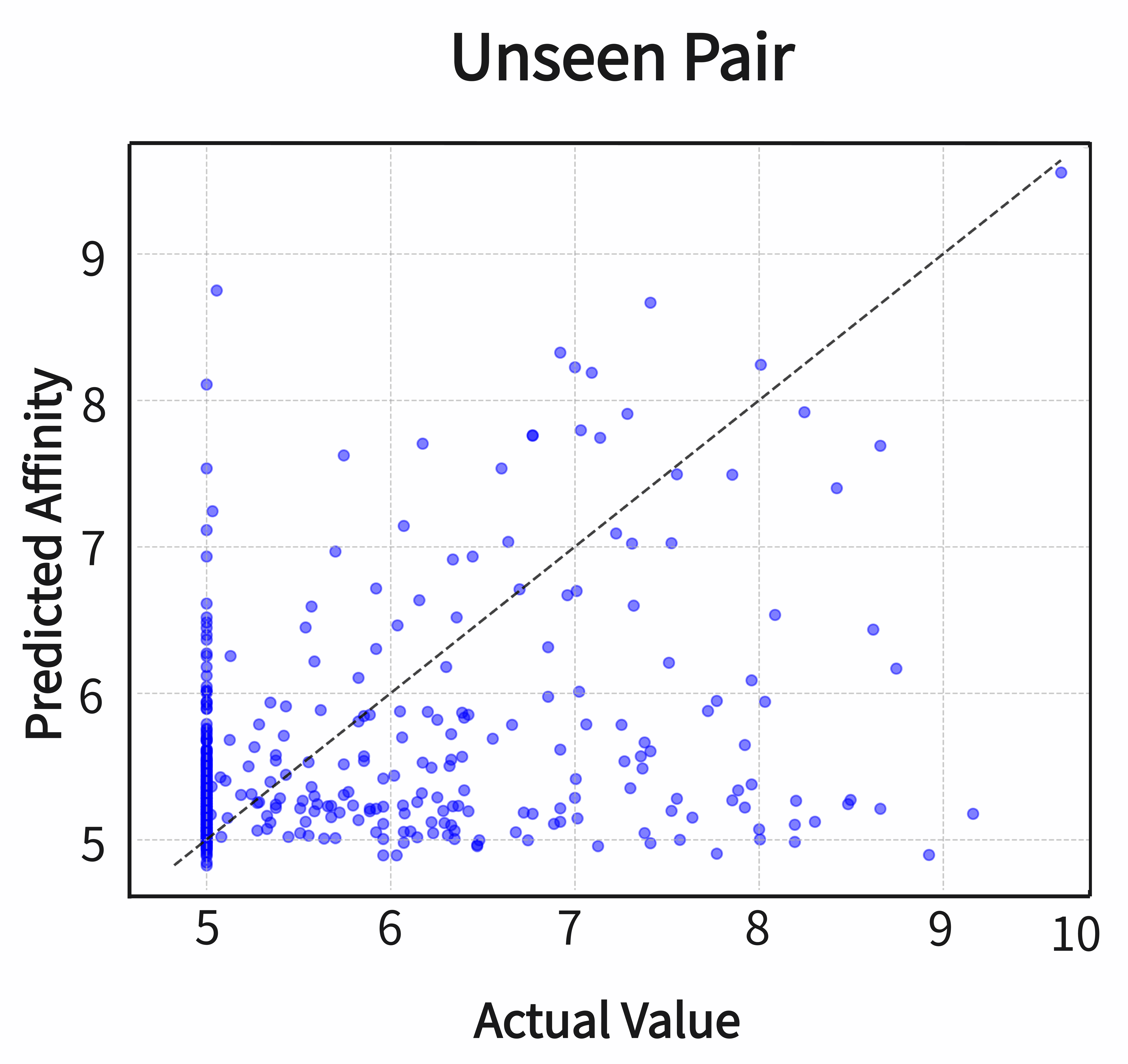}
		\label{fig:sub3}
	\end{subfigure}

      \vspace{0.5em}
	
	\begin{subfigure}[b]{0.3\textwidth}
		\includegraphics[width=\textwidth]{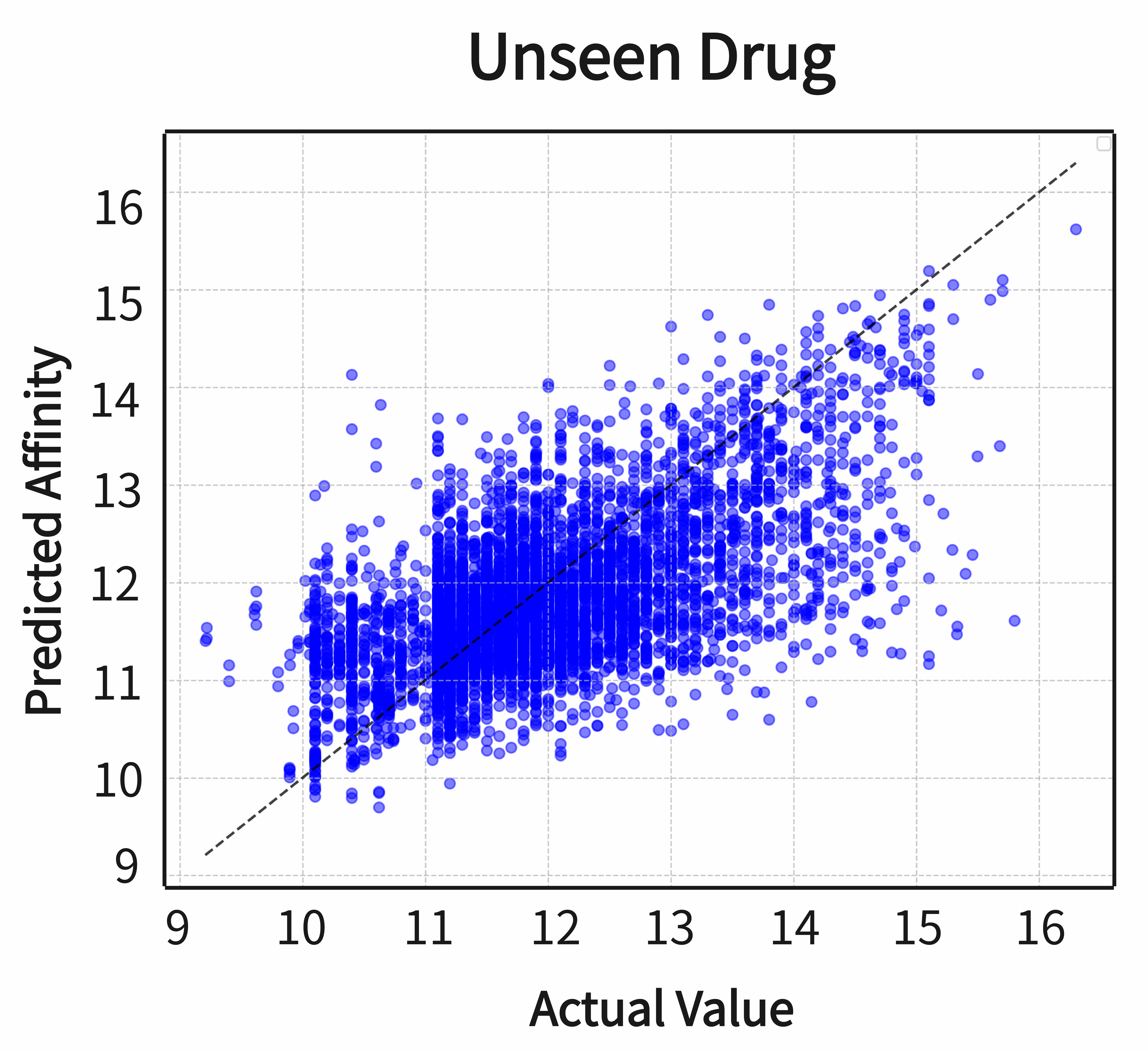}
		\label{fig:sub4}
	\end{subfigure}
	\hspace{0.5mm} 
	\begin{subfigure}[b]{0.3\textwidth}
		\includegraphics[width=\textwidth]{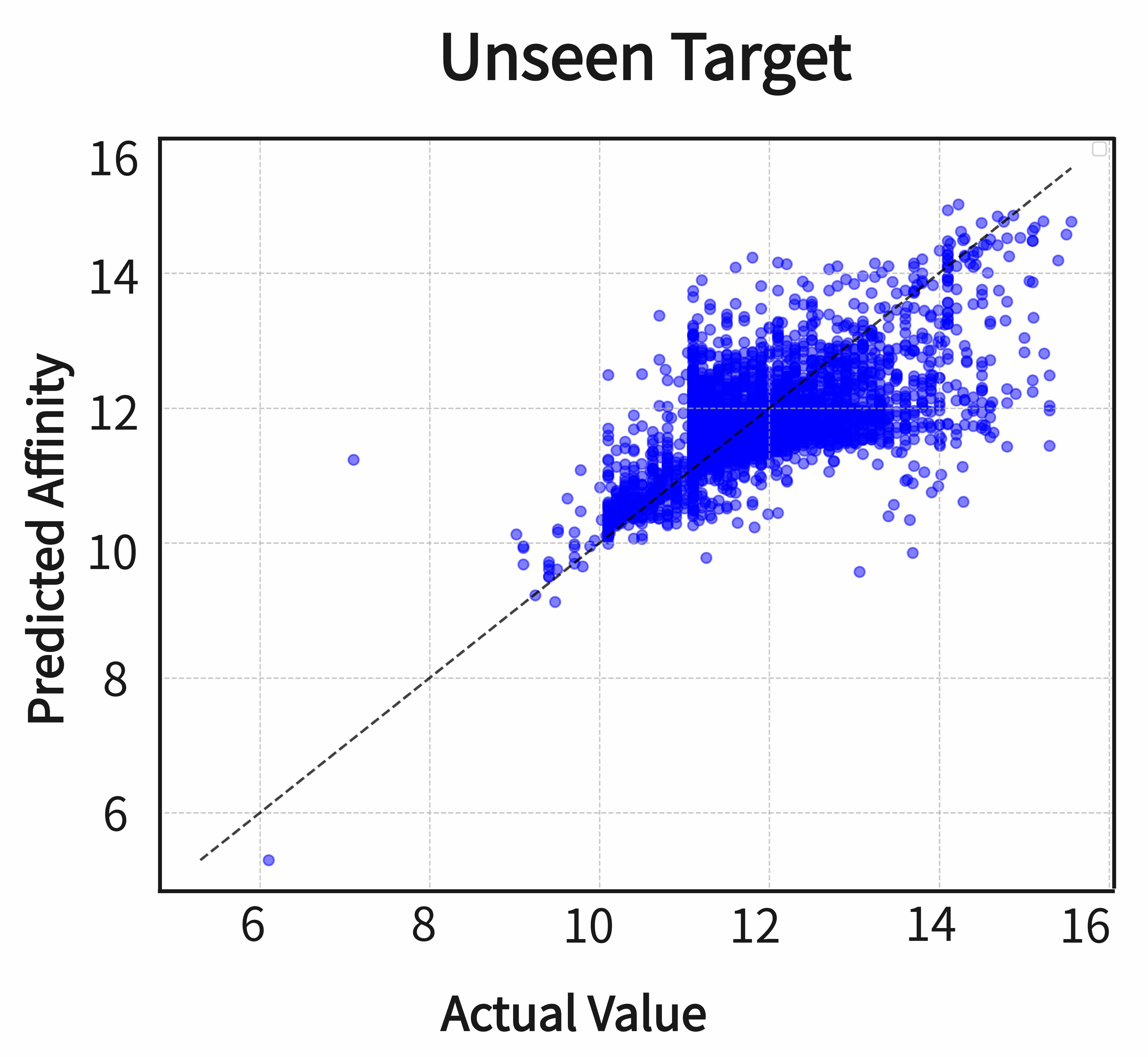}
		\label{fig:sub5}
	\end{subfigure}
	\hspace{0.5mm}
	\begin{subfigure}[b]{0.3\textwidth}
		\includegraphics[width=\textwidth]{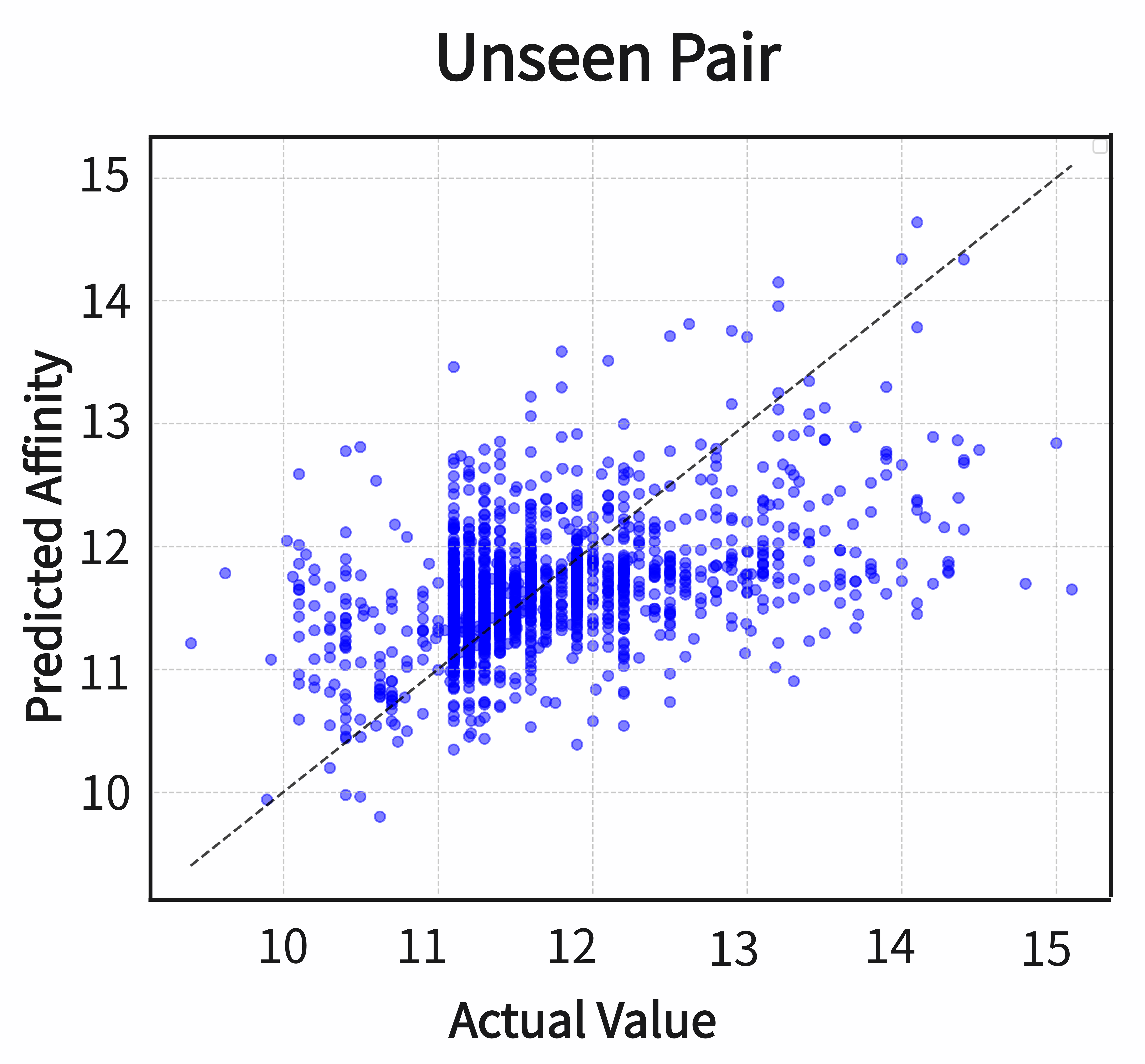}
		\label{fig:sub6}
	\end{subfigure}
	
	\caption{Scatter plots of predicted versus actual affinity on Davis (top) and KIBA (bottom) under Unseen Drug, Unseen Target, and Unseen Pair. The dashed line denotes the identity line.}
	\label{truth}
\end{figure*}

\subsection{Ablation Study}
\subsubsection{Mechanistic Contribution of the Diffusion Module} We assess the contribution of the modality-specific diffusion modules on the Davis dataset (Table~\ref{ablation}). We compare three variants: Drug-Diffusion (diffusion on the drug branch only), Target-Diffusion (target branch only), and No-Diffusion (variational latents without diffusion).

\begin{table}[htbp]
	\centering
	\caption{Ablation test of diffusion components under cold-start scenarios on the Davis dataset.}
	\label{ablation}
	\renewcommand{\arraystretch}{1.1}
	\setlength{\tabcolsep}{4pt}
	\begin{tabular}{p{1.5cm}lcccc}
		\toprule
		\textbf{Settings} & \textbf{Method} & \textbf{MSE}$\downarrow$  & \textbf{MAE}$\downarrow$  & \textbf{CI(std)}$\uparrow$ & \textbf{$r_m^2$}$\uparrow$ \\
		\midrule
		\multirow{4}{*}{Unseen Drug} 
		& Co-Diffusion       & \textbf{0.568} & \textbf{0.471} & \textbf{0.722} & \textbf{0.166} \\
		& Drug-Diffusion    & 0.580 & 0.498 & 0.630 & 0.122 \\
		& Target-Diffusion  & 0.617 & 0.537 & 0.636 & 0.122 \\
		& No-diffusion      & 0.571 & 0.519 & 0.619 & 0.131 \\
		\midrule
		\multirow{4}{*}{Unseen Target} 
		& Co-Diffusion      & 0.556 & \textbf{0.429} & \textbf{0.836} & 0.461 \\
		& Drug-Diffusion       & 0.542 & 0.482 & 0.821 & 0.457 \\
		& Target-Diffusion     & \textbf{0.529} & 0.498 & 0.835 & \textbf{0.484} \\
		& No-diffusion              & 0.570 & 0.518 & 0.811 & 0.443 \\
		\midrule
		\multirow{4}{*}{Unseen Pair} 
		& Co-Diffusion       & \textbf{0.806} & \textbf{0.551} & \textbf{0.728} & \textbf{0.149} \\
		& Drug-Diffusion       & 0.828 & 0.590 & 0.620 & 0.103 \\
		& Target-Diffusion    & 0.846 & 0.606 & 0.647 & 0.112 \\
		& No-diffusion      & 0.821 & 0.615 & 0.608 & 0.098 \\
		\bottomrule
	\end{tabular}
\end{table}

In the unseen-drug scenario, the full Co-Diffusion model attains the best results across all metrics, suggesting that dual-modality regularization is essential for capturing complex interaction landscapes. Interestingly, while Target-Diffusion performs well on MSE in the unseen-target setting, the full Co-Diffusion model yields the lowest MAE and highest CI, indicating a better global alignment of the latent space. In the unseen-pair scenario, the full model outperforms all variants, demonstrating that applying diffusion to both modalities provides the most consistent gains. This confirms that latent diffusion acts as a powerful perturb-and-denoise regularizer that prevents the model from overfitting to training-specific structural artifacts, validating our hypothesis that latent-space diffusion effectively regularizes the representation manifold against distribution shifts.

\begin{table*}[htbp] 
	\centering
	\caption{Comparison of Co-Diffusion trained with an affinity-aware two-stage regimen versus an end-to-end variant across cold-start scenarios on KIBA and Davis.}
	\label{two-stage}
	\renewcommand{\arraystretch}{1.1}
	\setlength{\tabcolsep}{5pt}
	\begin{tabular}{@{}ll|cccc|cccc@{}}
		\toprule
		\multirow{2}{*}{Setting} & \multirow{2}{*}{\begin{tabular}[c]{@{}l@{}}Method\\ (Co-Diffusion)\end{tabular}} & \multicolumn{4}{c|}{\textbf{KIBA}} & \multicolumn{4}{c}{\textbf{Davis}} \\
		\cmidrule(lr){3-6} \cmidrule(lr){7-10}
		& & \textbf{MSE}$\downarrow$ & \textbf{MAE}$\downarrow$ & \textbf{CI}$\uparrow$  & \textbf{RM2}$\uparrow$  & \textbf{MSE}$\downarrow$ & \textbf{MAE}$\downarrow$ & \textbf{CI}$\uparrow$  & \textbf{$r_m^2$}$\uparrow$  \\
		\midrule
		\multirow{2}{*}{Unseen Drug} 
		& end-to-end & 0.537 & 0.490 & 0.686 & 0.256 & 0.622 & 0.509 & 0.692 & 0.127 \\
		& two-stage & 0.421 & 0.406 & 0.751 & 0.402 & 0.568 & 0.471 & 0.722 & 0.166 \\
		\midrule
		\multirow{2}{*}{Unseen Target} 
		& end-to-end & 0.405 & 0.439 & 0.632 & 0.290 & 0.522 & 0.413 & 0.830 & 0.496 \\
		& two-stage & 0.347 & 0.383 & 0.685 & 0.377 & 0.556 & 0.429 & 0.836 & 0.461 \\
		\midrule
		\multirow{2}{*}{Unseen Pair} 
		& end-to-end & 0.488 & 0.505 & 0.597 & 0.135 & 0.858 & 0.570 & 0.693 & 0.121 \\
		& two-stage  & 0.432 & 0.459 & 0.660 & 0.201 & 0.806 & 0.551 & 0.728 & 0.149 \\
		\bottomrule
	\end{tabular}
\end{table*}

\subsubsection{Significance of the Affinity-Aware Two-Stage Scheme}
To evaluate the necessity of our affinity-steered Stage I, we compared Co-Diffusion against an "end-to-end" variant where the encoder, diffusion module, and regressor are optimized simultaneously. As shown in Table~\ref{two-stage}, the two-stage scheme consistently outperforms the end-to-end alternative across both datasets.

The performance gap arises from the resolution of the reconstruction--regression conflict. In end-to-end optimization, the model faces competing gradient pressures: the diffusion module seeks to maximize structural density (denoising fidelity), while the regressor demands focus on fine-grained affinity determinants. This often leads to semantic dilution, where the latent variables sacrifice predictive utility for structural reconstruction. In contrast, our scheme sequentially decouples these objectives—Stage I anchors the latent manifold to binding semantics, while Stage II refines this manifold through diffusion-based regularization. This ensures that the learned representations are both noise-tolerant and task-specific.

\subsection{Hyperparameter Sensitivity and Exploration-Exploitation Trade-off}
We investigate the sensitivity of Co-Diffusion to the number of steps $T$ and noise magnitude $\beta$ (Tables~\ref{sensitivity}, \ref{Beta}). All other settings follow our default configuration to isolate the effect of each hyperparameter.

\subsubsection{Impact of Diffusion Steps $T$}
We vary $T\in{800,1000,1200}$. Performance follows a concave profile, with $T{=}1000$ offering the optimal equilibrium between distributional exploration and denoising stability. On KIBA, MAE is minimized near $T{=}1000$ across unseen-drug/target/pair splits, while $r_m^2$ peaks or remains near-peak. On Davis, unseen-target shows a monotonic increase in $r_m^2$ with $T$, whereas MAE is lowest at $T{=}1000$. At $T=800$, the forward process under-explores the latent space, limiting the model's ability to bridge domain shifts. Conversely, at $T=1200$, the extended reverse chain tends to accumulate stochastic errors, marginally degrading the $r_m^2$ and MAE. A minor exception occurs on the small-scale Davis unseen-pair split, where a shorter chain ($T=800$) slightly mitigates noise accumulation, suggesting that step counts should be modestly tuned according to data density.

\subsubsection{Impact of Noise Magnitude $\beta$}

We next probed the noise magnitude by evaluating three linear $\beta$ ranges: $[1\mathrm{e}{-5},,4\mathrm{e}{-5}]$, $[1\mathrm{e}{-4},,4\mathrm{e}{-4}]$, and $[1\mathrm{e}{-4},,4\mathrm{e}{-3}]$. The linear $\beta$ range $[1\mathrm{e}{-4},,4\mathrm{e}{-4}]$ proved to be the most reliable schedule. Insufficient noise ($[1\mathrm{e}{-5},,4\mathrm{e}{-5}]$) results in weak latent perturbations, failing to push the model beyond the training manifold. Excessive noise ($[1\mathrm{e}{-4},,4\mathrm{e}{-3}]$), however, dominates the latent signal, making the reverse process numerically unstable and causing a sharp decline in $r_m^2$ (e.g., from 0.1491 to 0.0950 on Davis unseen-pair). Mechanistically, a moderate $\beta$ balances manifold expansion with recoverability, ensuring that the diffusion process regularizes without destroying the underlying binding semantics.

\begin{table}[htbp]
\centering
\caption{Sensitivity of Co-Diffusion to the number of diffusion steps $T$ on Davis and KIBA.}
\label{sensitivity}
\renewcommand{\arraystretch}{1.1}
\setlength{\tabcolsep}{6pt}
\begin{tabular}{c l cc cc}
\toprule
\multirow{2}{*}{$T$} & \multirow{2}{*}{Setting} & \multicolumn{2}{c}{Davis} & \multicolumn{2}{c}{KIBA} \\
\cmidrule(lr){3-4}\cmidrule(lr){5-6}
& & MAE$\downarrow$ & {$r_m^2$}$\uparrow$ & MAE$\downarrow$ & {$r_m^2$}$\uparrow$\\
\midrule
\multirow{3}{*}{800}
& Unseen Pair & 0.5487 & 0.1401 & 0.4669 & 0.2074 \\
& Unseen Drug & 0.4709 & 0.1614 & 0.4096 & 0.4011 \\
& Unseen Target   & 0.4311 & 0.4477 & 0.3953 & 0.3801 \\
\midrule
\multirow{3}{*}{1000}
& Unseen Pair & 0.5510 & 0.1491 & 0.4591 & 0.2011 \\
& Unseen Drug & 0.4710 & 0.1660 & 0.4065 & 0.4016 \\
& Unseen Target & 0.4293 & 0.4609 & 0.3831 & 0.3772 \\
\midrule
\multirow{3}{*}{1200}
& Unseen Pair & 0.5926 & 0.1363 & 0.4679 & 0.2055 \\
& Unseen Drug  & 0.4989 & 0.1565 & 0.4107 & 0.3920 \\
& Unseen Target  & 0.4431 & 0.4736 & 0.4038 & 0.3708 \\
\bottomrule
\end{tabular}
\end{table}

\begin{table}[htbp]
\centering
\caption{Sensitivity of Co-Diffusion to the number of noise magnitude $beta$ on Davis and KIBA.}
\label{Beta}
\renewcommand{\arraystretch}{1.1}
\setlength{\tabcolsep}{5pt}
\begin{tabular}{l l cc cc}
\toprule
\multirow{2}{*}{\makebox[2cm][c]{Beta}}
& \multirow{2}{*}{Setting}
& \multicolumn{2}{c}{Davis}
& \multicolumn{2}{c}{KIBA} \\
\cmidrule(lr){3-4}\cmidrule(lr){5-6}
& & MAE$\downarrow$ & {$r_m^2$}$\uparrow$ & MAE$\downarrow$ & {$r_m^2$}$\uparrow$ \\
\midrule
\multirow{3}{*}{$[1\mathrm{e}{-5},\, 4\mathrm{e}{-5}]$}
& Unseen Pair& 0.5842 & 0.1408 & 0.4713 & 0.1986 \\
& Unseen Drug  & 0.5100 & 0.1612 & 0.4062 & 0.3967 \\
& Unseen Target  & 0.4845 & 0.4317 & 0.4048 & 0.3557 \\
\midrule
\multirow{3}{*}{$[1\mathrm{e}{-4},\, 4\mathrm{e}{-4}]$}
& Unseen Pair& 0.5510 & 0.1491 & 0.4591 & 0.2011 \\
& Unseen Drug   & 0.4710 & 0.1660 & 0.4065 & 0.4016 \\
& Unseen Target  & 0.4293 & 0.4609 & 0.3831 & 0.3772 \\
\midrule
\multirow{3}{*}{$[1\mathrm{e}{-4},\, 4\mathrm{e}{-3}]$}
& Unseen Pair & 0.5717 & 0.0950 & 0.4886 & 0.1645 \\
& Unseen Drug   & 0.4807 & 0.1193 & 0.4538 & 0.3435 \\
& Unseen Target  & 0.5042 & 0.2946 & 0.4285 & 0.3399 \\
\bottomrule
\end{tabular}
\end{table}


\subsection{Latent Manifold Visualization and Topological Expansion}

\begin{figure*}[htpb]
	\centering
	\begin{subfigure}[b]{0.3\textwidth}
		\includegraphics[width=\textwidth]{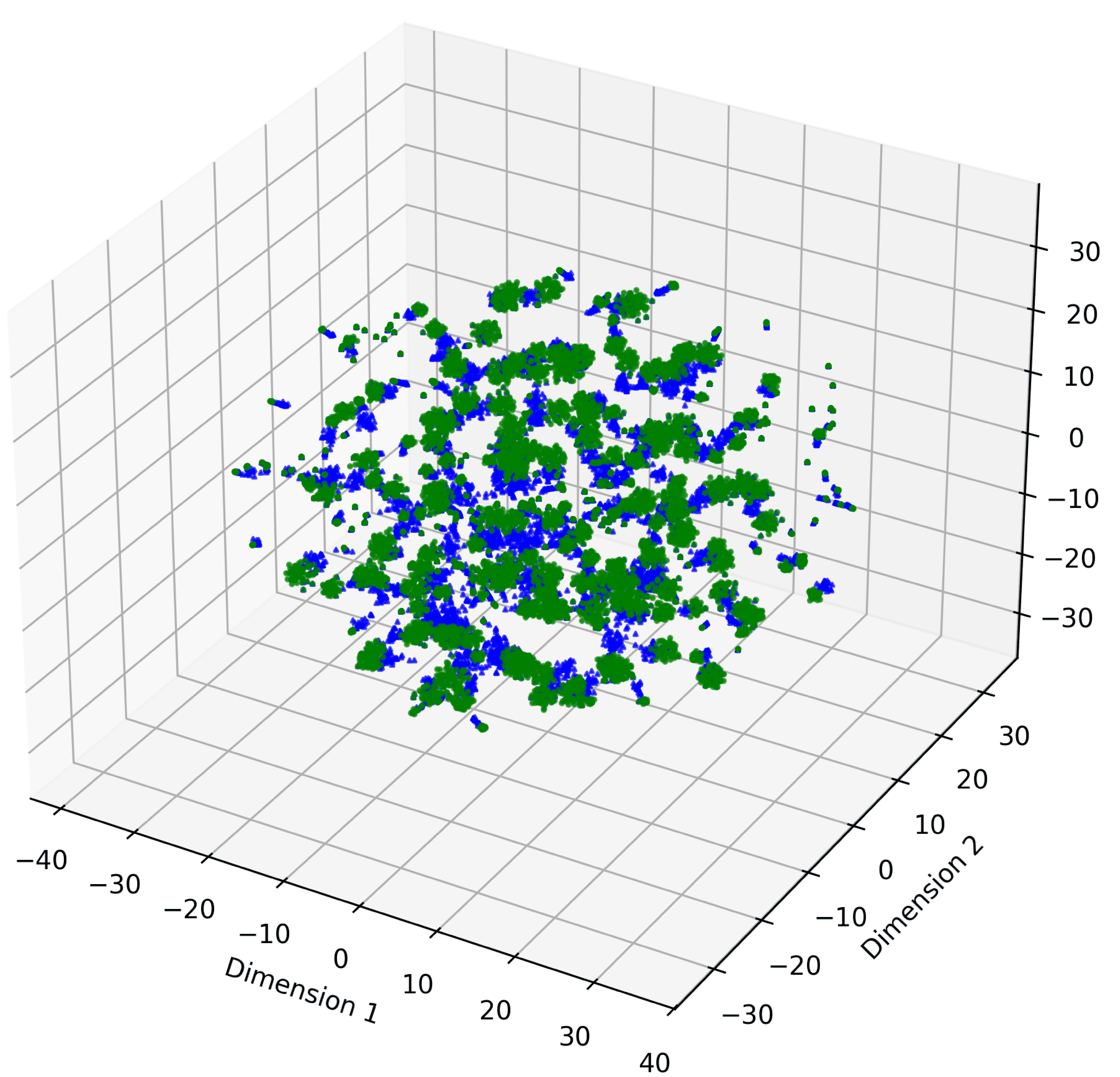}
		\caption{New vs. Old Drugs (Unseen Drug)}
		\label{fig:sub1}
	\end{subfigure}
	\hspace{0.5mm}
	\begin{subfigure}[b]{0.3\textwidth}
		\includegraphics[width=\textwidth]{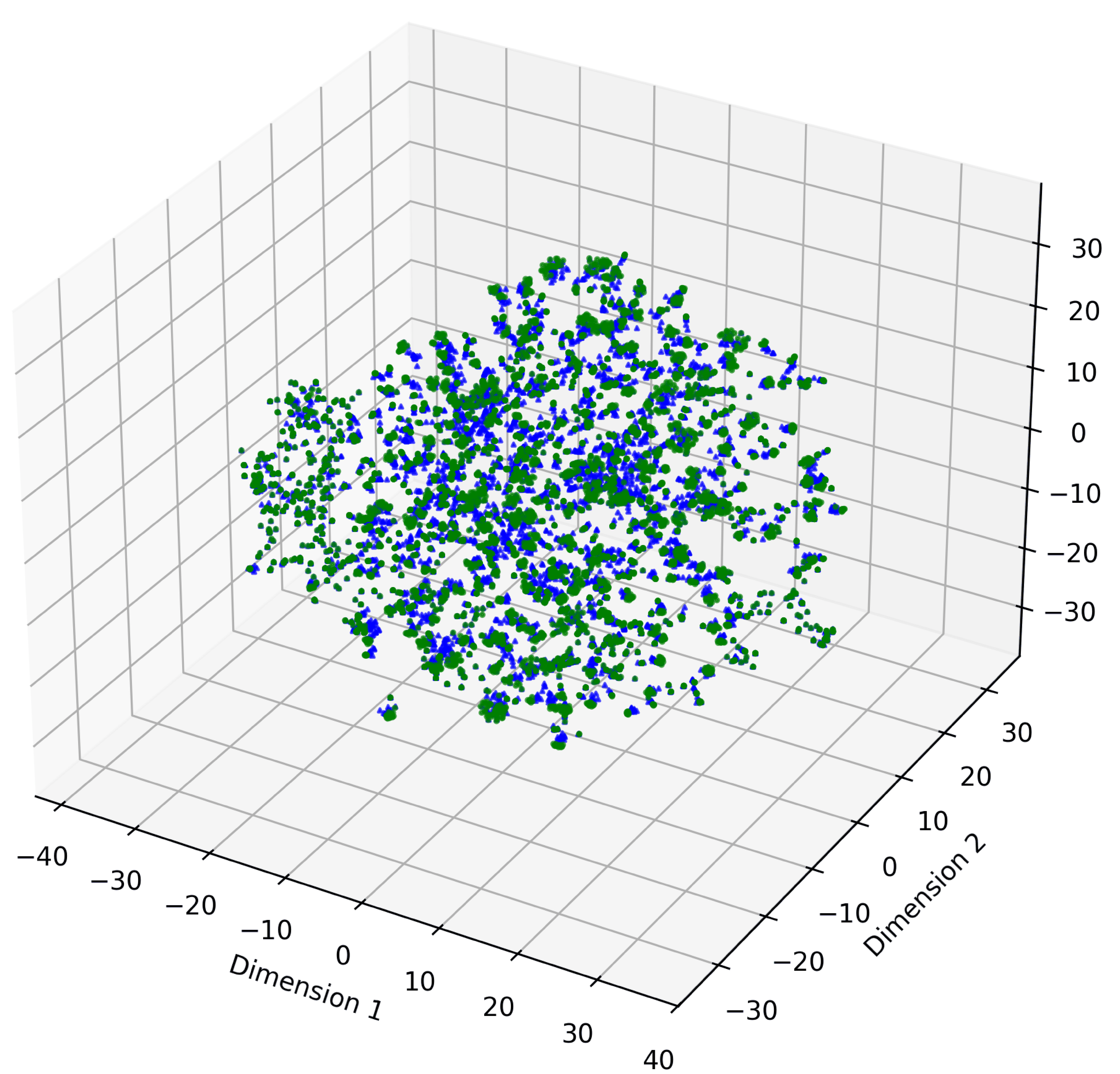}
		\caption{New vs. Old Drugs (Unseen Target)}
		\label{fig:sub2}
	\end{subfigure}
	\hspace{0.5mm}
	\begin{subfigure}[b]{0.3\textwidth}
		\includegraphics[width=\textwidth]{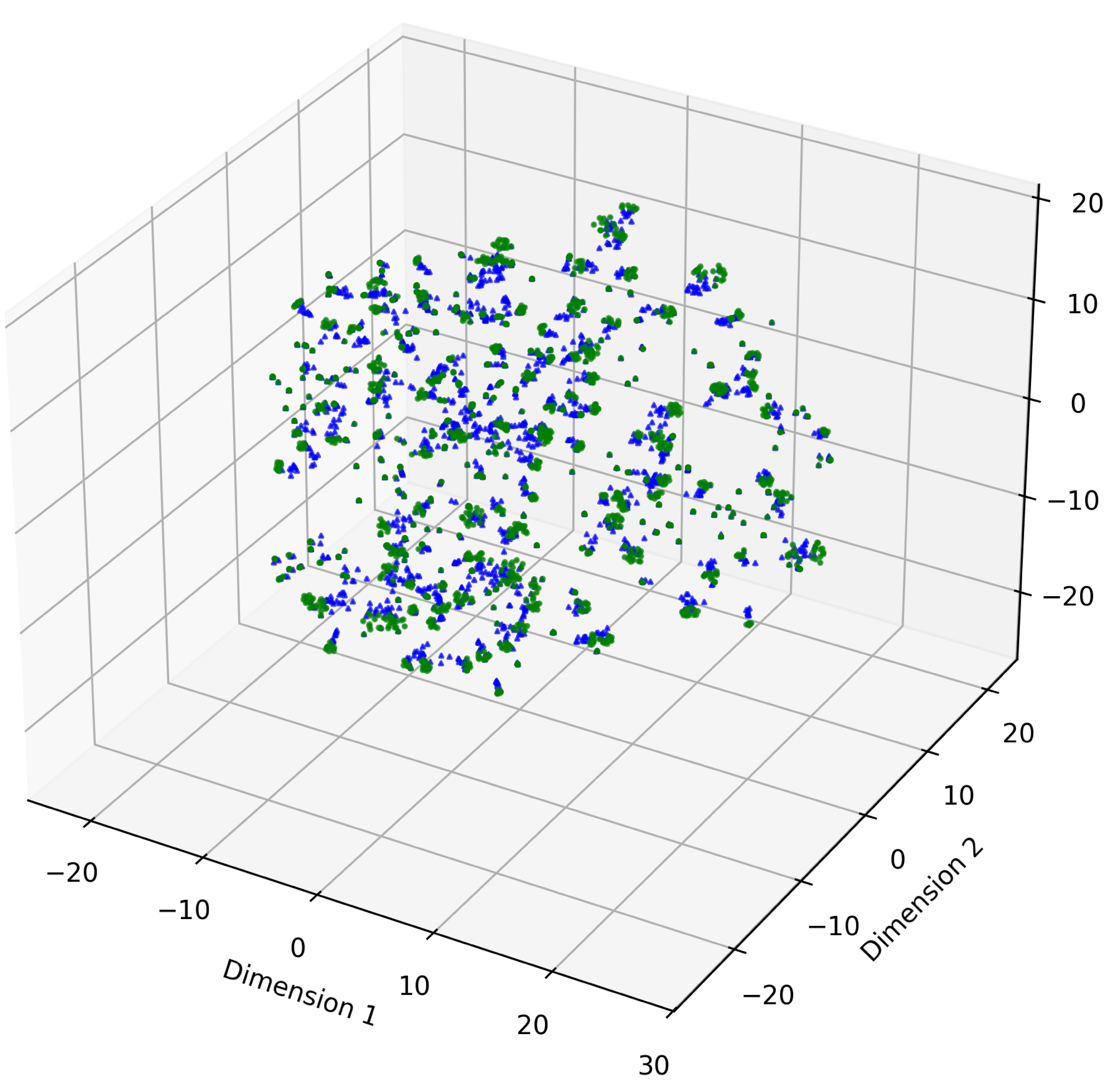}
		\caption{New vs. Old Drugs (Unseen Pair)}
		\label{fig:sub3}
	\end{subfigure}

	\begin{subfigure}[b]{0.3\textwidth}
		\includegraphics[width=\textwidth]{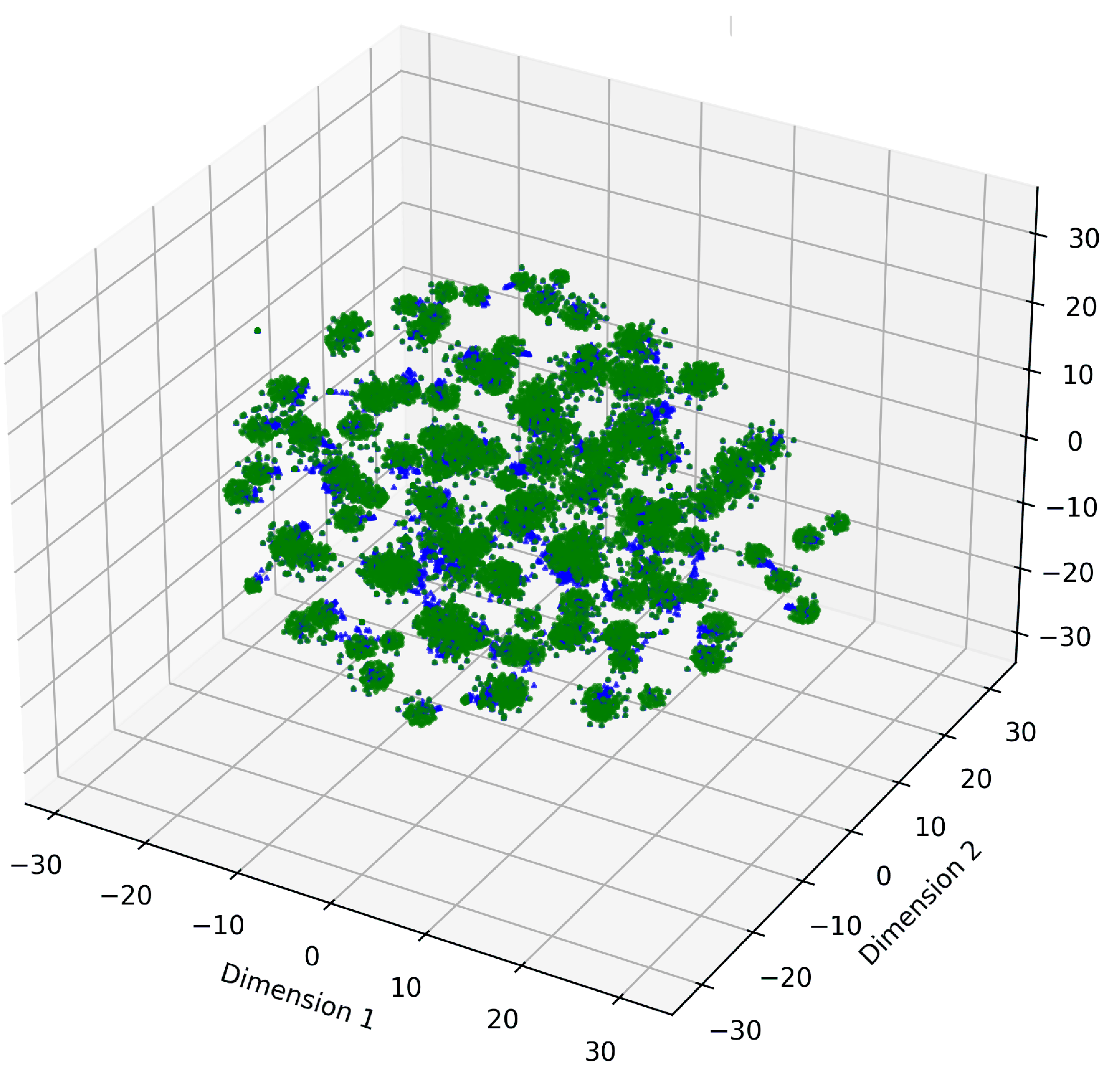}
		\caption{New vs. Old Targets (Unseen Drug)}
		\label{fig:sub4}
	\end{subfigure}
	\hspace{0.5mm} 
	\begin{subfigure}[b]{0.3\textwidth}
		\includegraphics[width=\textwidth]{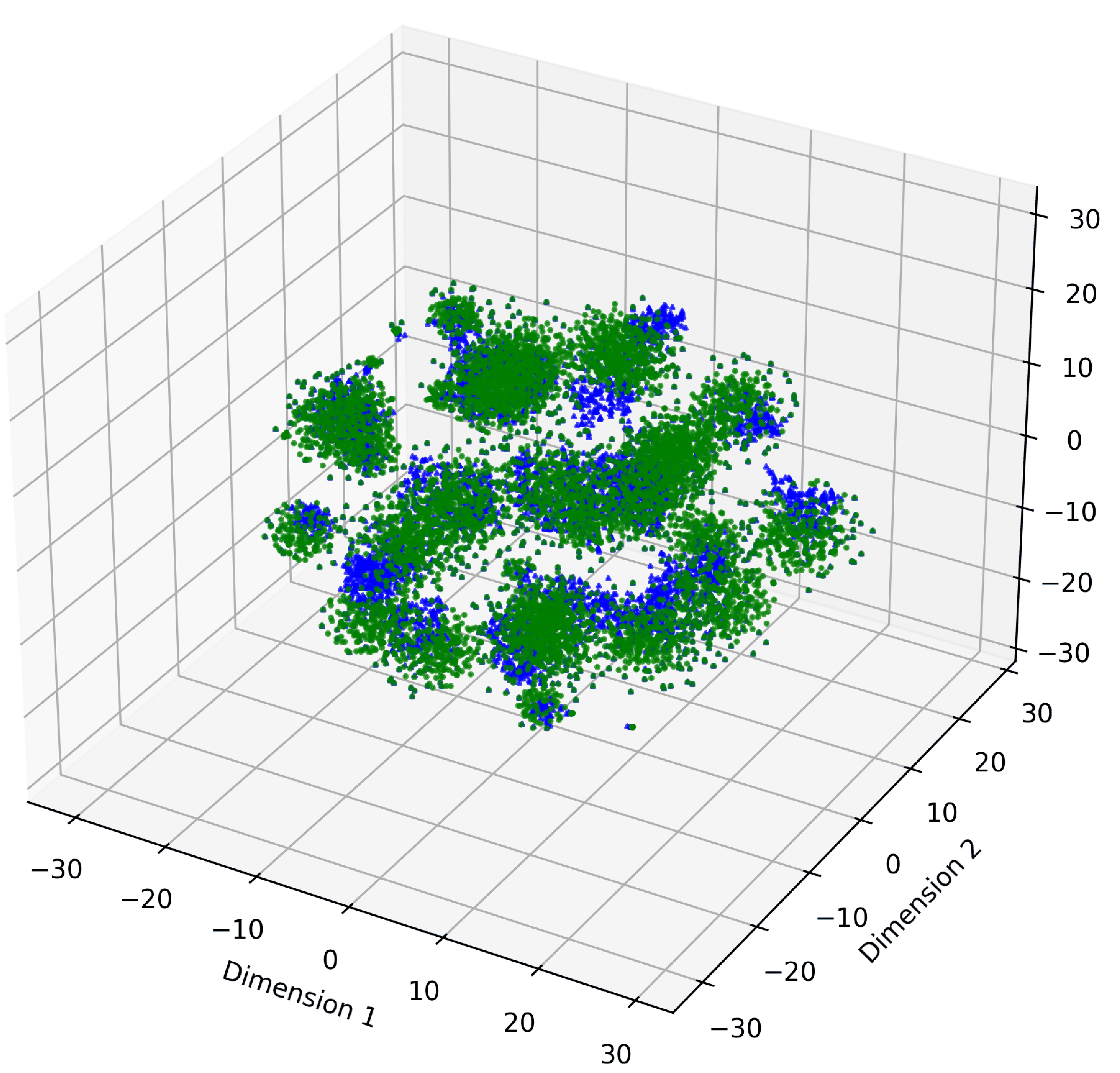}
		\caption{New vs. Old Targets (Unseen Target)}
		\label{fig:sub5}
	\end{subfigure}
	\hspace{0.5mm}
	\begin{subfigure}[b]{0.3\textwidth}
		\includegraphics[width=\textwidth]{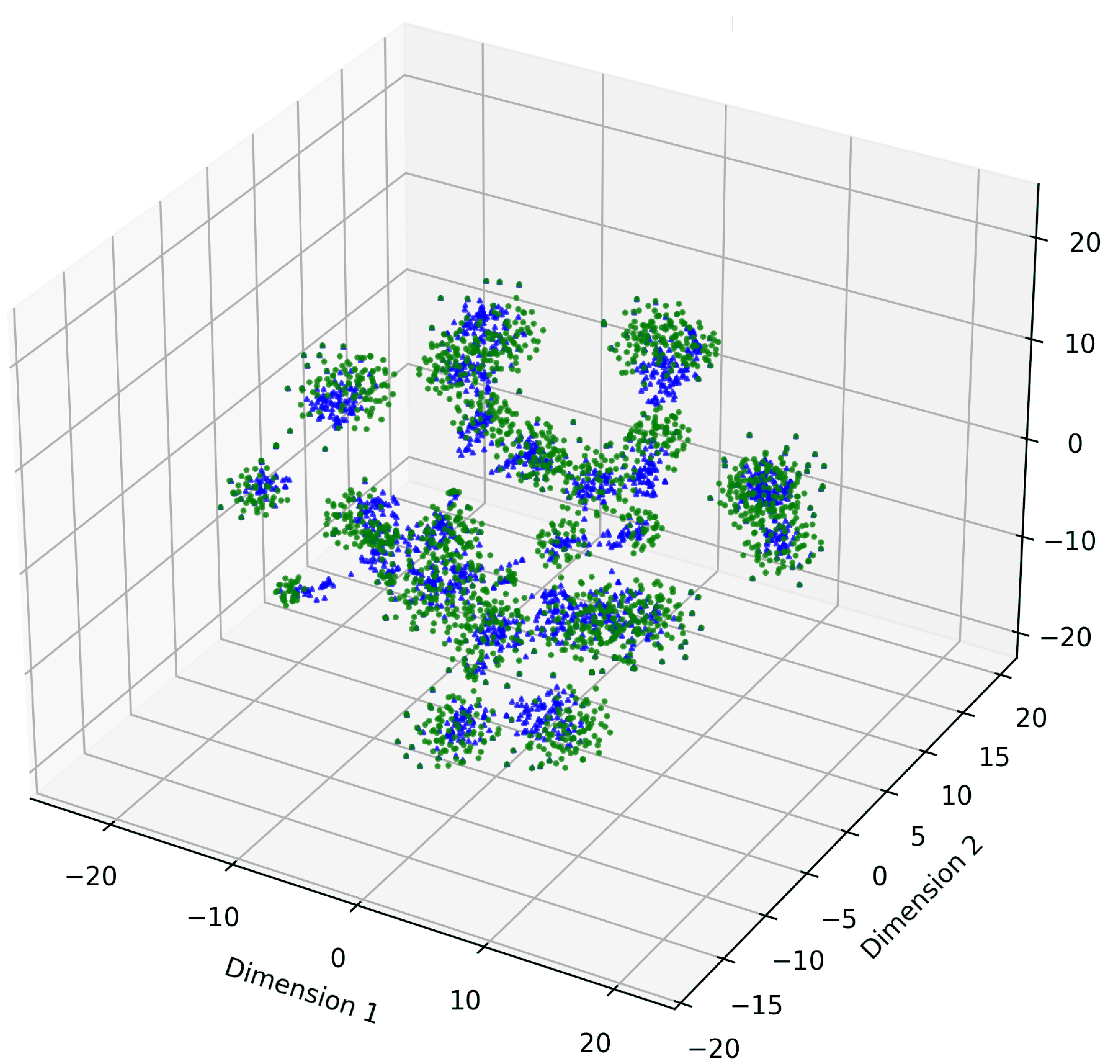}
		\caption{New vs. Old Targets (Unseen Pair)}
		\label{fig:sub6}
	\end{subfigure}

	\begin{subfigure}[b]{0.3\textwidth}
		\includegraphics[width=\textwidth]{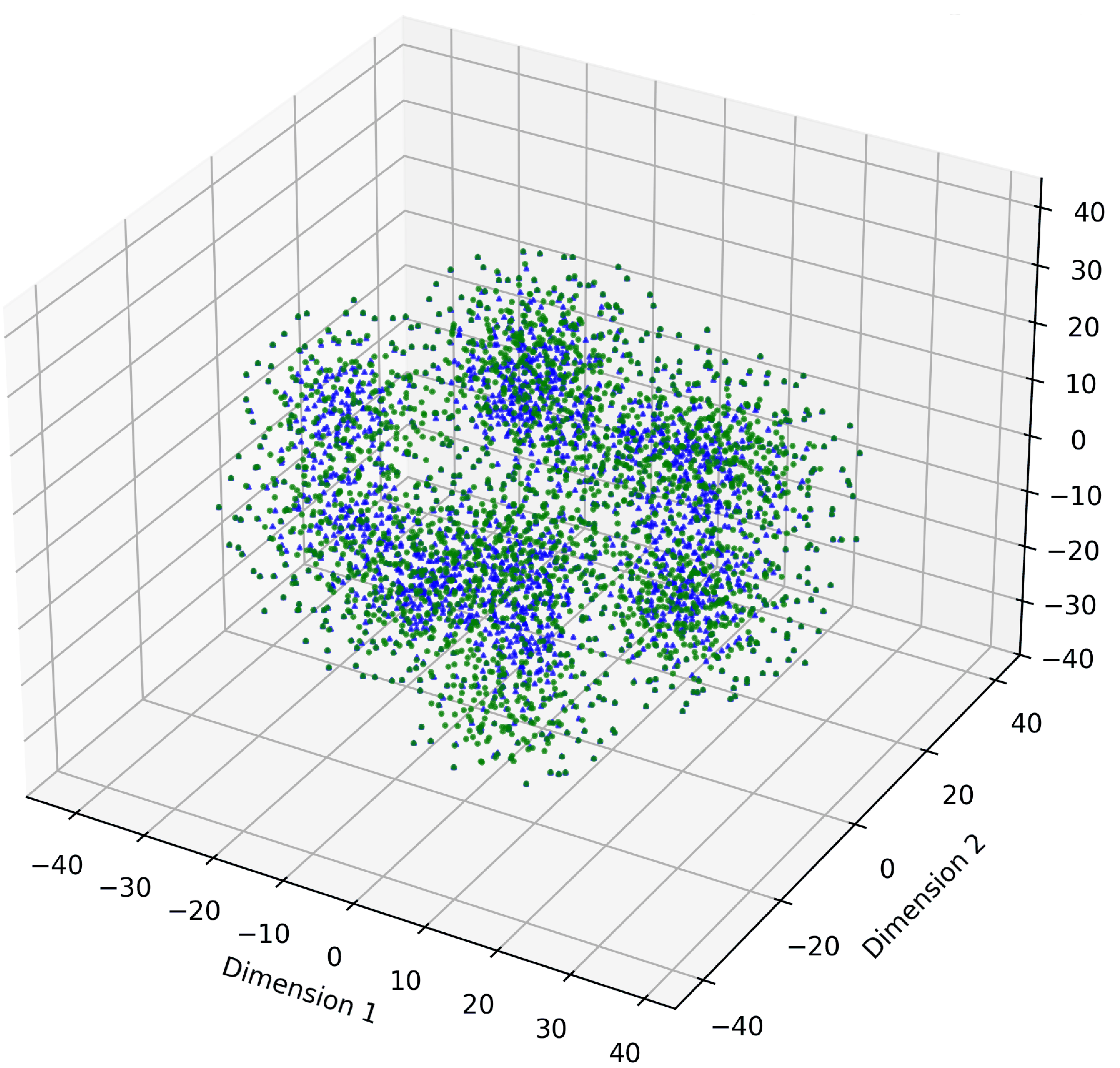}
		\caption{New vs. Old Drugs (Unseen Drug)}
		\label{fig:sub7}
	\end{subfigure}
	\hspace{0.5mm}
	\begin{subfigure}[b]{0.3\textwidth}
		\includegraphics[width=\textwidth]{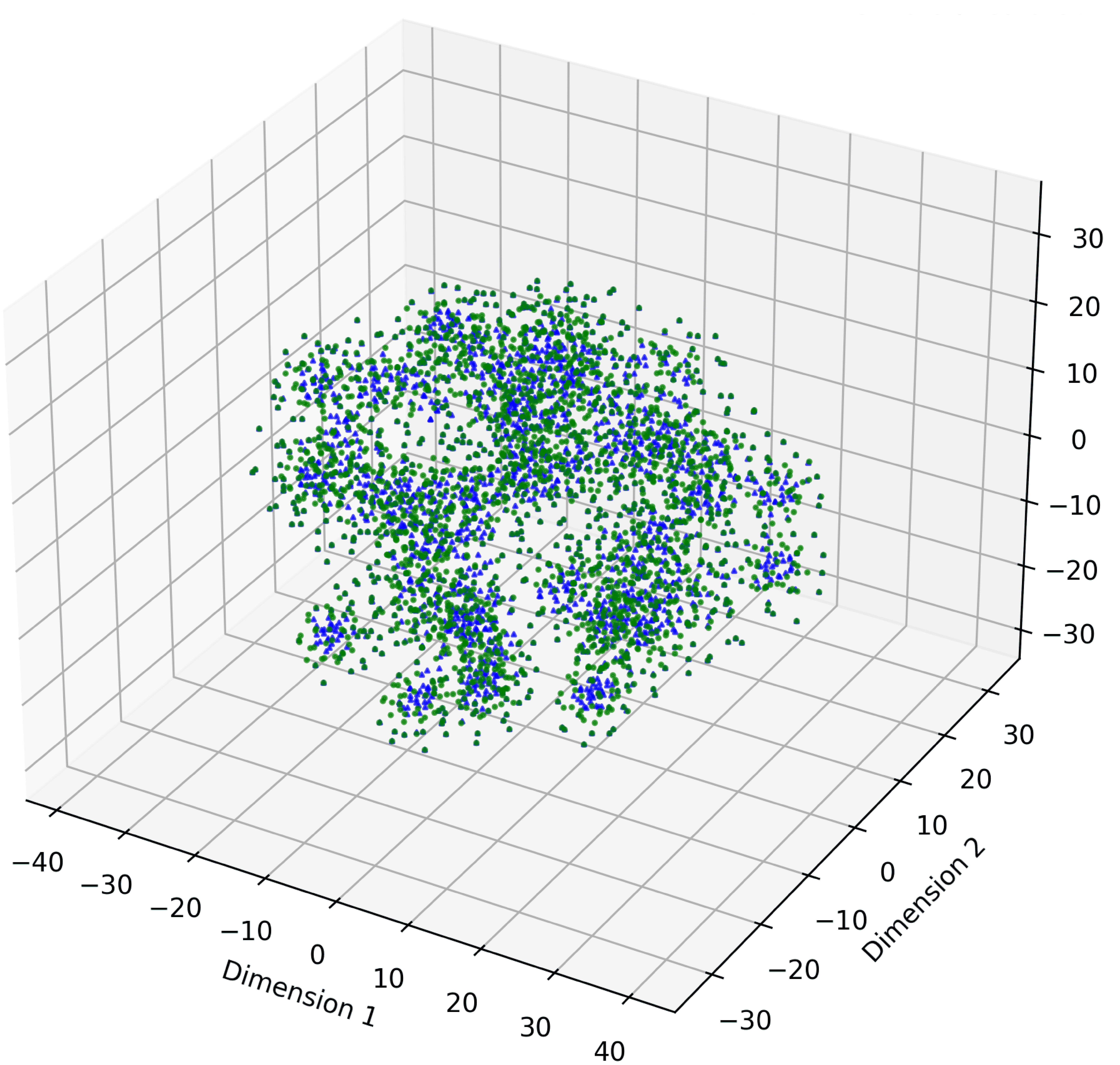}
		\caption{New vs. Old Drugs (Unseen Target)}
		\label{fig:sub8}
	\end{subfigure}
	\hspace{0.5mm}
	\begin{subfigure}[b]{0.3\textwidth}
		\includegraphics[width=\textwidth]{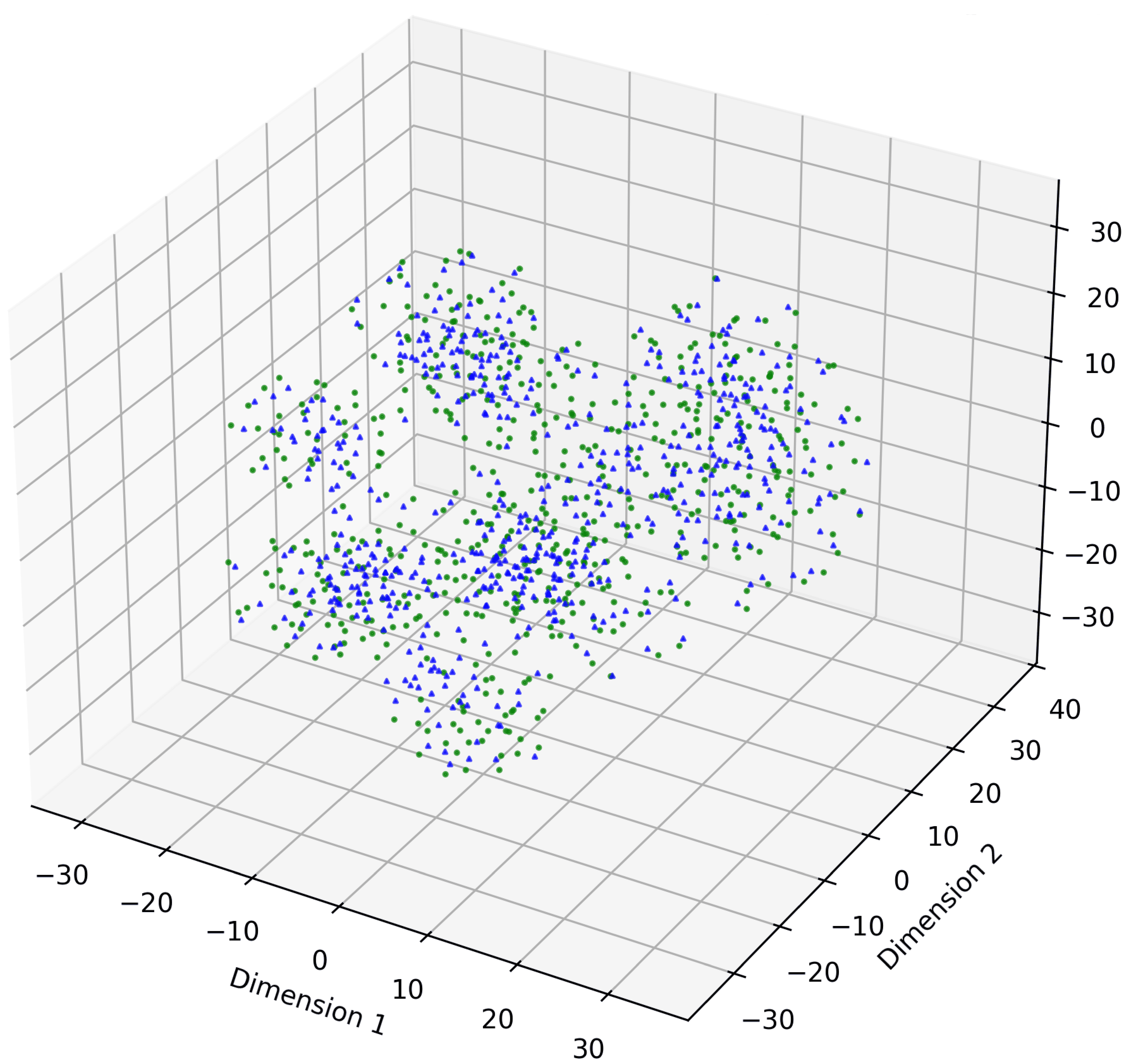}
		\caption{New vs. Old Drugs (Unseen Pair)}
		\label{fig:sub9}
	\end{subfigure}

	\begin{subfigure}[b]{0.3\textwidth}
		\includegraphics[width=\textwidth]{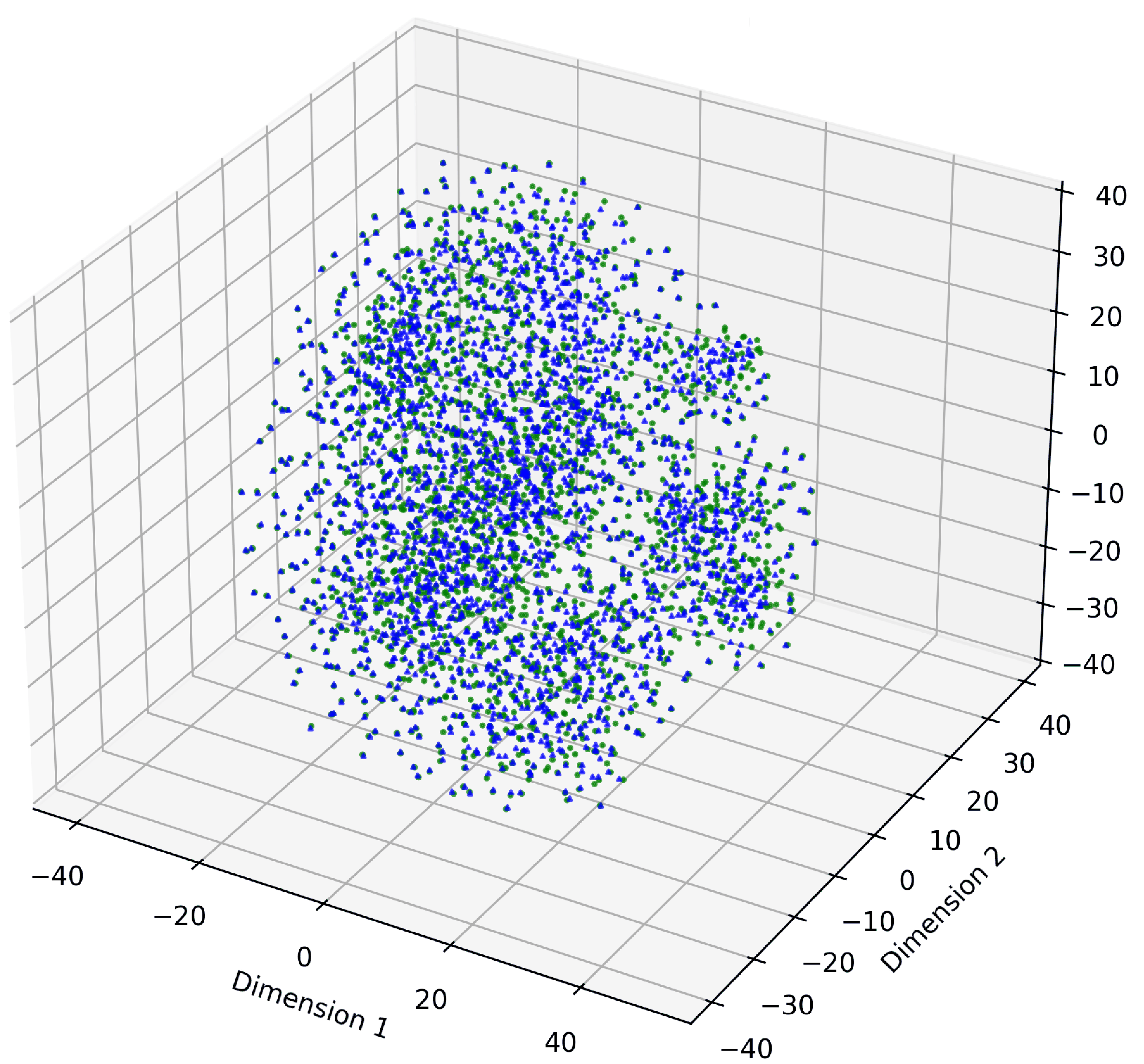}
		\caption{New vs. Old Targets (Unseen Drug)}
		\label{fig:sub10}
	\end{subfigure}
	\hspace{0.5mm} 
	\begin{subfigure}[b]{0.3\textwidth}
		\includegraphics[width=\textwidth]{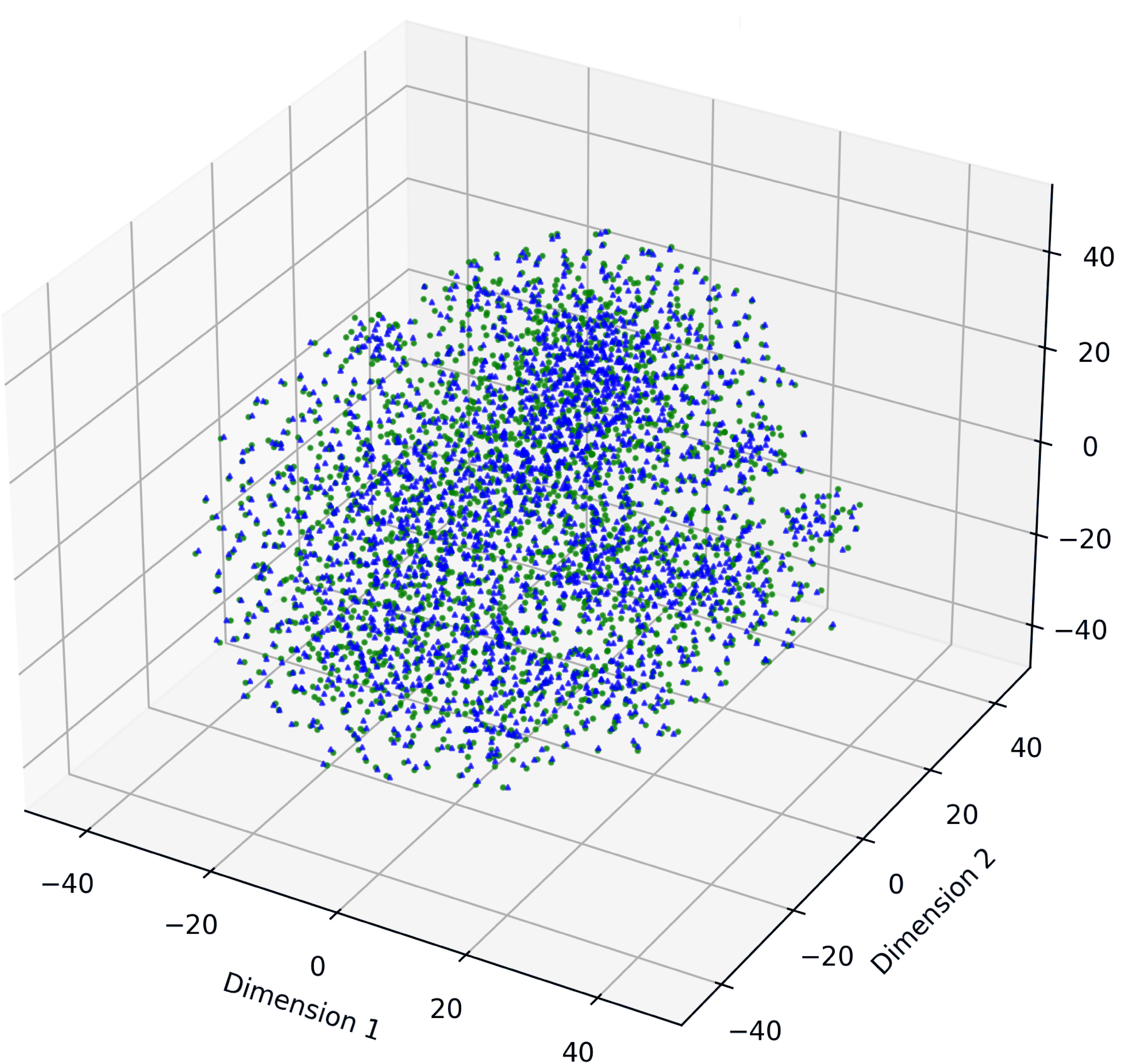}
		\caption{New vs. Old Targets (Unseen Target)}
		\label{fig:sub11}
	\end{subfigure}
	\hspace{0.5mm}
	\begin{subfigure}[b]{0.3\textwidth}
		\includegraphics[width=\textwidth]{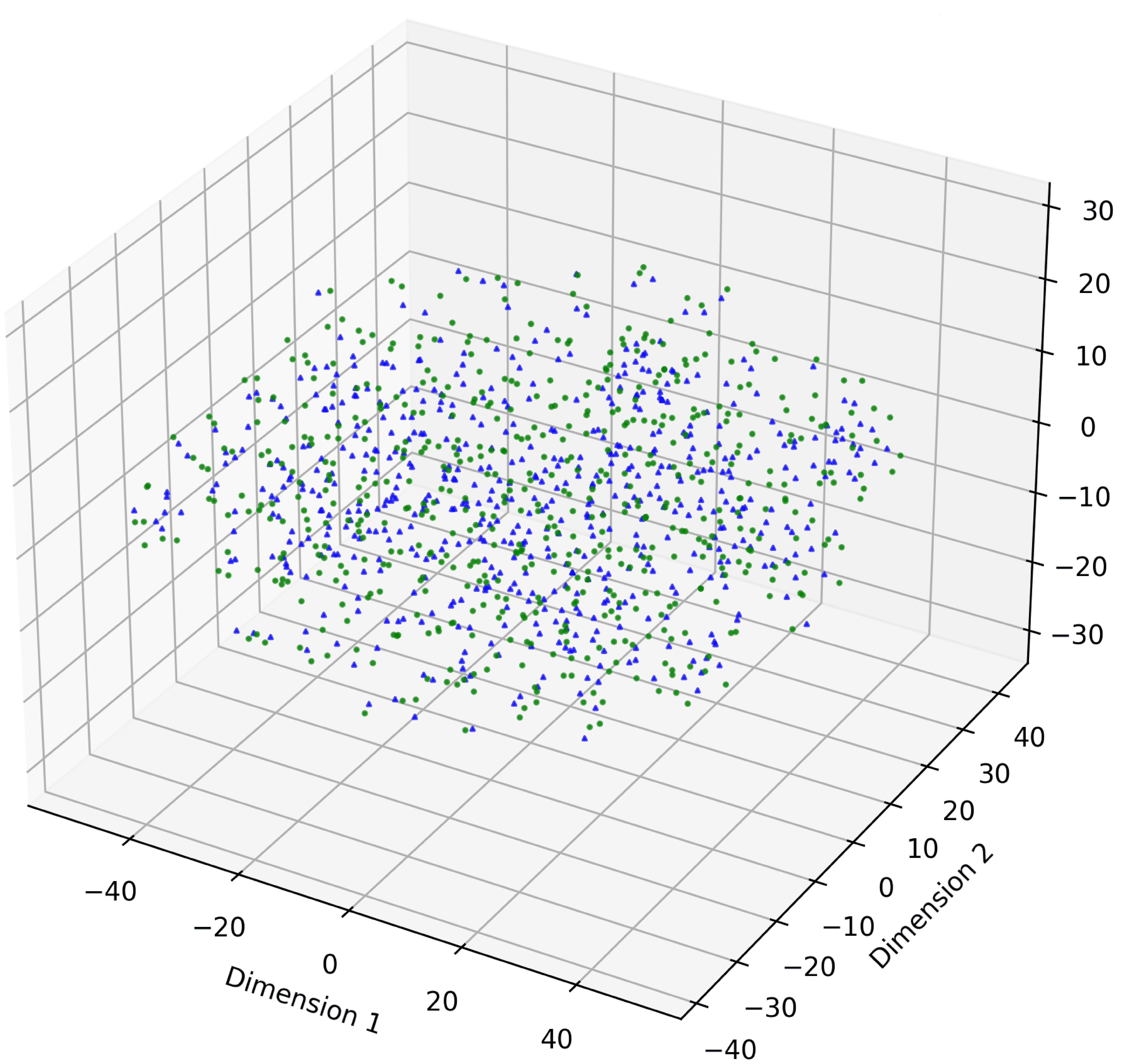}
		\caption{New vs. Old Targets (Unseen Pair)}
		\label{fig:sub12}
	\end{subfigure}
	
	\caption{t-SNE projections of new (blue) vs. original (green) drug and target embeddings on KIBA (a–f) and Davis (g–l) under unseen drug, unseen target, and unseen pair scenarios.}
	\label{new}
\end{figure*}

To probe how Co-Diffusion enhances generalization, we visualize latent distributions of original and refined embeddings using t-SNE (Fig.\ref{new}). Panels (a--f) for KIBA and (g--l) for Davis show that diffusion-refined embeddings (blue) do not merely replicate original features (green) but strategically expand into the sparse topological voids of the latent space. This structured expansion indicates that Co-Diffusion avoids "outlier" scatter; instead, it populates under-represented regions of the manifold with noise-robust representations. Such behavior is consistent with the model interpolating the binding landscape, which effectively bridges the gap between seen training scaffolds and unseen test chemotypes. The overlap between original and refined clusters confirms that our affinity-consistency constraint successfully preserves biological plausibility while fostering the diversity required for robust cold-start prediction.

\begin{table}[t]
\centering
\caption{Out-of-sample results on recent PDBbind entries.}
\label{pdbbind}
\renewcommand{\arraystretch}{1}
\setlength{\tabcolsep}{12pt}
\begin{tabular}{lcccc}
\toprule
\textbf{ID} & \textbf{Year} & \textbf{True} & \textbf{Pred.} & \textbf{MSE} ($\downarrow$) \\
\midrule
6j0g & 2019 & 5.640 & 5.609 & 0.001 \\
6ce6 & 2018 & 5.100 & 5.162 & 0.004 \\
5zyg & 2019 & 5.210 & 5.277 & 0.005 \\
6qlp & 2019 & 5.280 & 5.349 & 0.005 \\
6s56 & 2019 & 5.100 & 5.182 & 0.007 \\
5xpi & 2018 & 5.280 & 5.189 & 0.008 \\
6j0o & 2019 & 4.920 & 5.018 & 0.010 \\
6ajy & 2019 & 5.280 & 5.379 & 0.010 \\
5ja0 & 2017 & 5.220 & 5.120 & 0.010 \\
6dq4 & 2018 & 4.920 & 5.023 & 0.011 \\
6n3y & 2019 & 5.440 & 5.331 & 0.012 \\
5wxh & 2017 & 6.150 & 6.034 & 0.013 \\
6i13 & 2019 & 5.120 & 5.261 & 0.020 \\
6t1i & 2019 & 5.260 & 5.104 & 0.024 \\
6pgc & 2019 & 5.350 & 5.190 & 0.026 \\
5u0e & 2017 & 5.110 & 5.283 & 0.030 \\
6t1m & 2019 & 5.210 & 5.394 & 0.034 \\
6gfz & 2018 & 5.510 & 5.321 & 0.036 \\
5uk8 & 2017 & 6.400 & 6.603 & 0.041 \\
6i12 & 2019 & 5.570 & 5.356 & 0.046 \\
\vdots & \vdots & \vdots & \vdots & \vdots \\
        
\midrule
\multicolumn{5}{c}{\textit{Overall Performance (Mean MSE)}} \\
\midrule
\multicolumn{4}{l}{Pair-VAE (2025)} & 1.179 \\
\multicolumn{4}{l}{\textbf{Co-Diffusion (Ours)}} & \textbf{0.961} \\
\bottomrule
\end{tabular}
\end{table}
\subsection{Out-of-Sample Test}

To evaluate Co-Diffusion’s generalization beyond standard benchmarks, we performed an out-of-sample evaluation on “fresh” data. This set was curated from recently deposited affinity measurements in the PDBbind database, ensuring that all targets and drugs had zero overlap with the training distribution. This setup mimics a prospective drug discovery scenario where the model must navigate unexplored regions of the drug-target manifold.

As summarized in Table~\ref{pdbbind}, Co-Diffusion achieves an overall Mean Squared Error (MSE) of 0.961, significantly outperforming the latest generative baseline, PAIR-VAE (MSE = 1.179). The substantial performance gap (an 18.5\% improvement) highlights the inherent limitations of standard variational frameworks in handling domain shifts. While PAIR-VAE relies on heuristic pair-synthesis to regularize its latent space, its representations often suffer from semantic dilution when encountering truly out-of-distribution samples. In contrast, Co-Diffusion leverages the expressive prior-learning of the latent diffusion module to recover robust binding determinants even under significant structural perturbations.

A detailed error analysis reveals that Co-Diffusion provides highly calibrated predictions for a substantial proportion of the external set: 34\% of the entries exhibit a per-sample squared error of $\leq 0.1$, and 45\% remain within $\leq 0.2$. The Top-20 ranked predictions, spanning various protein families and molecular scaffolds, demonstrate that the model maintains high predictive fidelity across diverse biochemical contexts (Table~\ref{pdbbind}).


\subsection{Discussion}

Through comprehensive benchmarking across diverse split protocols, Co-Diffusion demonstrates three distinctive advantages that address the intrinsic bottlenecks of DTA prediction. \textbf{(1) Robust Generalization via Manifold Interpolation.} Unlike purely discriminative models that often succumb to the "shortcut learning" of observed pairs, Co-Diffusion exhibits superior performance in cold-start scenarios. Our results suggest that by regularizing the latent space through a denoising process, the model implicitly performs manifold interpolation—populating under-represented regions of the drug-target landscape with noise-robust, transferable binding determinants. This allows for stable predictions even when encountering molecular scaffolds or protein families that lie far from the training distribution. \textbf{(2) Synergizing Affinity-Awareness with Generative Priors.} A pivotal innovation of Co-Diffusion is the two-stage decoupling of alignment and refinement. While traditional VAE-based approaches often suffer from semantic dilution due to the reconstruction-regression conflict, our Stage~I anchors the latent space to binding semantics, while Stage~II employs latent diffusion as a stochastic perturb-and-denoise regularizer. This design ensures that the generative power of diffusion is strictly harnessed to refine affinity-relevant features rather than merely reconstructing structural noise. The observed t-SNE clusters confirm that diffusion-refined embeddings maintain high semantic fidelity while fostering the diversity required for out-of-distribution robustness. \textbf{(3) Theoretical Rigor and Probabilistic Coherence.} Beyond empirical success, Co-Diffusion is grounded in a principled probabilistic framework. By optimizing a variational lower bound on the joint likelihood of drug, target, and affinity, the framework provides a mathematical guarantee that the denoising refinement remains consistent with the underlying binding physics. This coherence bridges the gap between black-box deep learning and interpretable probabilistic inference.

\subsection{Limitations and Future Work}

Despite its promising performance, several avenues remain for the enhancement of Co-Diffusion. First, the current framework utilizes a uniform diffusion horizon and a predefined noise schedule. In the context of heterogeneous chemical and protein data, implementing adaptive or learnable noise schedules could further balance regularization strength with sampling efficiency. Second, while Co-Diffusion achieves high fidelity using 1D/2D inputs, incorporating 3D geometric priors or pocket-specific conformations could further enrich the latent space, particularly for complex binding sites with high conformational plasticity. Future research will focus on (i) geometry-aware conditional diffusion to inject physics-based constraints into the reverse process and (ii) exploring the model's performance in ultra-large-scale virtual screening to assess its real-world throughput and reliability in pharmaceutical lead discovery.


\section{CONCLUSION}
In this paper, we tackle the critical challenge of cold-start generalization in DTA prediction by introducing Co-Diffusion, a novel affinity-aware latent diffusion framework. By decoupling the prediction task into an affinity-steered alignment stage and a diffusion-based refinement stage, Co-Diffusion effectively reconciles the expressive capacity of generative priors with the precision of binding-relevant semantics. This architecture enables the model to recover robust binding determinants from stochastic perturbations, significantly mitigating the performance degradation caused by label scarcity and domain shift. Theoretically, we demonstrate that the framework optimizes a principled variational lower bound on the joint distribution of drugs and targets. Empirically, extensive evaluations across multiple benchmarks—including prospective PDBbind validation—show that Co-Diffusion consistently outstrips existing discriminative and variational baselines. Ultimately, Co-Diffusion offers a robust and theoretically grounded paradigm for computational triage, paving the way for more accurate and efficient exploration of the vast chemical space in early-stage drug discovery.


{\small

}

\end{document}